\documentclass[final,12pt]{alt2021} 

\usepackage{hyperref}       
\usepackage{url}            
\usepackage{booktabs}       
\usepackage{amsfonts}       
\usepackage{nicefrac}       
\usepackage{microtype}      
\usepackage{graphicx}
\usepackage{algorithm}
\usepackage{algpseudocode}
\usepackage{enumerate}
\usepackage{enumitem}
\usepackage{algorithmicx}
\usepackage{bm}

\usepackage{%
	anyfontsize,%
	etoolbox,%
	mathtools,%
	pgf,%
	pgfplots,%
	stmaryrd,%
	tikz,%
	xcolor,%
	tabto,%
    comment
}
\usepackage[capitalise]{cleveref}

\newtheorem{assumption}{Assumption}

\usepgfplotslibrary{patchplots}
\usepgflibrary{shapes}
\usetikzlibrary{%
	arrows,%
	decorations.text,%
	positioning,%
	scopes,%
	shapes%
}

\newcommand{\NAM}{\textsc{CMAB-SM}}
\renewcommand{\cite}[1]{\citep{#1}}

\usepgflibrary{shapes}
\usetikzlibrary{%
	arrows,%
	decorations.text,%
	positioning,%
	scopes,%
	shapes%
}
\pgfplotsset{compat=1.15}

\begin{document}
\title{Stochastic Top-$K$ Subset Bandits with Linear Space and Non-Linear Feedback}
\altauthor{%
 \Name{Mridul Agarwal} \Email{agarw180@purdue.edu}\\
 \addr Purdue University, West Lafayette, IN, USA 
 \AND
 \Name{Vaneet Aggarwal} \Email{vaneet@purdue.edu}\\
 \addr Purdue University, West Lafayette, IN, USA 
 \AND
 \Name{Christopher J. Quinn} \Email{cjquinn@iastate.edu}\\
 \addr Iowa State University, Ames, IA, USA
 \AND
 \Name{Abhishek K. Umrawal} \Email{aumrawal@purdue.edu}\\
 \addr Purdue University, West Lafayette, IN, USA 
}

\maketitle

\begin{abstract}
Many real-world problems like Social Influence Maximization face the dilemma of choosing the best $K$ out of $N$ options at a given time instant. This setup can be modeled as a combinatorial bandit which chooses $K$ out of $N$ arms at each time, with an aim to achieve an efficient trade-off between exploration and exploitation. This is the first work for combinatorial bandits where the feedback received can be a non-linear function of the chosen $K$ arms.  The direct use of multi-armed bandit requires choosing among $N$-choose-$K$ options making the state space large. In this paper, we present a novel algorithm which is computationally efficient and the storage is linear in $N$. The proposed algorithm is a divide-and-conquer based strategy, that we call CMAB-SM. Further, the proposed algorithm achieves a \textit{regret bound} of $\tilde O(K^{\frac{1}{2}}N^{\frac{1}{3}}T^{\frac{2}{3}})$ for a time horizon $T$, which is \textit{sub-linear} in all parameters $T$, $N$, and $K$. 
\end{abstract}

\section{Introduction}


Multi-Armed Bandits (MAB) can be used to solve problems in domains where an agent chooses an arm to play at each time instant and receives a reward. The goal of the agent is to perform online learning to select the best arm as early as possible after some initial exploration. However, many real-world problems are combinatorial in nature, where the agent chooses $K$ out of $N$ arms at each time and receives an aggregate reward. For example, the problem of Social Influence Maximization where the aim is to select a subset of individuals in a social network to adopt a new product or innovation, and the target is to trigger a large cascade of further adoptions \citep{domingos2001mining}. We have studied it later in great detail. Other applications include daily advertising campaign characterized by a set of sub-campaigns \citep{zhang2012joint,nuara2018combinatorial}, erasure-coded storage \citep{xiang2016joint} where $K$ out of $N$ servers are chosen to obtain the content for each request,  and cross-selling item selection with $K$ items in the bundle \citep{1250942}. 

We consider the setting where the agent plays a complex action made by choosing $K$ out of $N$ arms at each time and only receives an aggregate reward. This reward is a function of the rewards of different arms. This is an instance of Combinatorial Multi-Armed Bandit (CMAB) problem. Based on the history of actions played and rewards observed, the action at the next time is taken. This paper aims to find an efficient  algorithm to minimize the gap between the reward obtained by choosing the best $K$ arms at each time and the arms chosen by the CMAB algorithm, for a given time horizon $T$. Only an aggregate reward, a possibly non-linear function of individual arms' rewards, is observed.  
The non-linear function makes it hard to estimate the rewards of each arm individually. Thus,  algorithms that either assume access to, or first estimate, the individual arm rewards do not directly work \citep{pmlr-v28-chen13a,pmlr-v32-lind14}.  To our knowledge, this is the first work that proposes an algorithm for the CMAB problem with non-linear aggregate rewards without any extra feedback. 

One way for solving CMAB with non-linear feedback is to treat every action (consisting of $K$ arms) as an arm and apply the classical MAB framework using Upper Confidence Bound (UCB) algorithm \citep{auer2010ucb}. However, as the number of arms $N$ and the subset size $K$ increases, the total number of actions increases exponentially. Consider a case where $N=30$, $K=15$, and playing each action costs $1$ second. The time to explore all possible actions a single time will take {\color{black}${{30}\choose{15}}$ seconds or} approximately $5$ years. Moreover, to store all the possible actions and their rewards requires large memory, which might not be possible for applications with limited storage constraints \citep{pmlr-v84-liau18a}. This paper proposes a novel algorithm, called \NAM , which explores only $O(N)$ actions, and achieves a regret bound of $\tilde O(K^\frac{1}{2}N^\frac{1}{3}T^\frac{2}{3})$ \footnote{$\tilde{O}()$ hides the logarithmic factors.}. We note that using the UCB approach as mentioned, the regret bound will be $O\Big(\sqrt{{{N}\choose{K}}T}\Big)$, which is better than the proposed bound of $O(T^{2/3})$ for small values of $N$ and $K$. However, this is not the case when $N$ and $K$ are large. Again for the case of $N=30$, and $K=15$, the regret bounds of \NAM\ will match that of 
UCB only when $T$ is approximately $4\times 10^{26}$. Further, \NAM\ has a storage complexity of $O(N)$, and per round time complexity of 
{\color{black}$\tilde{O}(K)$}, making it efficient. 

\subsection{Our Contributions}

The main contributions of this paper can be summarized as follows:
\begin{enumerate}[leftmargin=*]
    \item We propose \NAM , the first, efficient algorithm for the CMAB problem when the reward is non-linear and no additional feedback about individual arms is available.  

    \item We use the theory of stochastic dominance to obtain ordering of individual arms instead of estimating arm distributions directly, which alleviates the need of linear feedback. This approach may be of independent interest for MAB problems. 

    \item We prove that \NAM\  achieves a regret of $\widetilde{O}({K^\frac{1}{2}N^\frac{1}{3}T^\frac{2}{3}})$ when the exploration-exploitation procedure runs for time horizon $T$. Thus the regret is \emph{sub-linear} with respect to each parameter.  

    \item We prove the algorithm is efficient, with time complexity of $O(TK\log K)$ and space complexity of $O(N)$.

    \item We apply \NAM\ to some synthetic problems and show that it outperforms the UCB algorithm. 

    
    
    
    
\end{enumerate}

\subsection{Key techniques}  We now summarize the method and proof techniques used.  \NAM\ divides all $N$ arms into groups of $K+1$ arms, such that each group contains only $K+1$ actions as there are $K$ arms to choose from $K+1$ arms. Since choosing $K$ out of $K+1$ is equivalent to removing $1$ out of $K+1$, there is a one-to-one mapping between arms and actions in a group. We sort the actions (and the corresponding arms) in each group which requires time steps of polynomial order in $K$. We then merge those groups one by one and obtain the best $K$ arms. 

For the analysis, 
we assume certain properties on the reward distributions of the different arms, and on the function of rewards of individual arms. More precisely, we use the theory of stochastic dominance to differentiate between cumulative distribution functions. The reward distributions of different arms are assumed to dominate or be dominated by each other. Further, the non-linear function is assumed to be symmetric in the rewards obtained from each arm, and the mean of the function is assumed to be continuous (in terms of dominated inputs). These properties are satisfied in case of a few reward distributions (e.g., Bernoulli rewards), and few non-linear functions (e.g., maximum).

\subsection{Related Work} \label{related_work}
Combinatorial Bandits have been studied where the agent chooses $K$ of the $N$ arms in each round  \citep{NIPS2007_3371,Cesa-Bianchi:2012:CB:2240304.2240495,Audibert:2014:ROC:2765232.2765234,dani08stochastic,abbasi-yadkori11improved}. In these works, it is assumed that the reward function in each round is linear in the different arms. They consider a setting where at time $t$, the agent selects an arm $x_t\in D_t$ and observes a reward $\theta^Tx_t$, where $D_t\subset\mathbb{R}^K$ is the decision set and $\theta\in\mathbb{R}^K$ is a constant vector. Due to this linear function, the problem is also called online linear optimization. The algorithms proposed in these works use the linearity of the reward function to estimate rewards of individual arms and achieve a regret of $\Tilde{O}\left(\sqrt{T}\right)$. Weights are then assigned to each of the $\binom{N}{K}$ actions to decide the action in the next round; such approaches are not computationally efficient for large $N$. The work of \citep{pmlr-v84-liau18a} reduces the space complexity at the cost of regret bounds. However, they still loop over all the arms and the regret bound becomes exponential in $K$.  For a linear function, such as sum of rewards, we can construct arms which is a binary $K$-sparse vector of length $N$. Such a setting has $N$ unknown variables, and those unknowns could be obtained using least squares as done by \cite{dani08stochastic} or regularized least squares \cite{abbasi-yadkori11improved}. We consider non-linear functions for which the expected rewards of individual arms could not be obtained using least squares solution.

\cite{filippi10parametric,jun17scalable,li17provably} studied the problem of generalized linear models (GLM) where reward $r_t$ is a function $\left(f:\mathbb{R}\to\mathbb{R}\right)$ of $z=\theta^T x_t$. Generalized linear models assume the distribution of bandit reward $r_t$ belongs to a cannonical exponential family. The exponential distribution allows the use of log-likelihood maximization to obtain estimates of arm parameters which increase the likelihood of observed rewards.  Generalized Linear Models assume that the expected reward of the arm played is a non linear function of the linear combination of features of the action played with a fixed parameter. This is different than our setup where we assume the reward of the action played is a non linear function of individual realization of rewards of each arm. 

GLM (and linear) models have long and rich history across many disciplines such as finding a target item among multiple options \citep{48839}. GLMs also have many interesting theoretical and statistical properties. But there are settings where GLMs do not accurately model rewards. For example, in the case where multiple arms are selected and the joint reward is the maximum of individual rewards, the joint rewards is not a linear combination of the individual arm rewards.  

\citep{kveton2014matroid} provides a UCB style algorithm for matroid bandits, where the agent selects a maximal independent set of rank $K$ to maximize sum of rewards of each arm. They assume rewards of each of the $K$ arms is also observed in each round. Such a setup, where the rewards of each of the $K$ arms is also available to agent, is referred to as a semi-bandit problem. \citep{gai2010learning} also considered the problem of semi-bandits for the problem of maximum weighted matching for cognitive radio applications. \citep{kveton2015tight} showed that the UCB algorithm  provides a tight regret bound for semi-bandit combinatorial bandit problem with linear reward function. The authors of \citep{pmlr-v28-chen13a} considered combinatorial semi-bandit problem with non-linear rewards  using a UCB style analysis. The authors of \citep{pmlr-v32-lind14} assumed combinatorial bandit problem with non-linear reward function and feedback, where the feedback  is a linear combination of rewards of the $K$ arms. Such feedback of linear function of rewards allows for the recovery of individual rewards.  In contrast to prior works, this paper does not consider the availability of individual arm rewards or a linear feedback.  With only aggregate, non-linear feedback, it might not be possible to obtain the exact values of the rewards of base arms.  





\subsection{Organization}

The rest of the paper is organized as follows. 
In Section \ref{formulation}, we provide the model under consideration and the assumptions that are taken for the analysis. Section \ref{algorithm_and_analysis} presents the proposed algorithm, CMAB-SM. The main result is provided in Section \ref{main_results}. Section \ref{sec:synth_eval} illustrates our results on a synthetic example where a continuous reward distribution and a non-linear reward function is chosen. 
Section \ref{conclusion} presents the conclusions with directions for future work.  

\section{Problem Formulation and Assumptions} \label{formulation}

\subsection{Problem Setup}
We now describe the  stochastic combinatorial multi-armed bandit problem we consider.  There are  $N$ ``arms'' labeled as $i \in {[N] = } \{1, 2, \cdots, N\}$.  Each time an arm is chosen or ``played,'' there is a reward. Let $X_{i,t} \in [0,1]$ be a random variable denoting the reward of the $i^{th}$ arm, at time-step $t$ (also referred to as time $t$). We assume that $X_{i,t}$ are independent across time and arms, and for any arm the distribution is identical at all times. {Also, if some analysis is independent of time variable, we will drop the subscript $t$ and write only $X_i$.}  The rewards could be discrete valued, continuous valued, or mixed. 

At each time instant, the agent chooses an action ${\bf a} = (a_1, a_2,\cdots, a_K)$ which is a $K$-tuple of arms. Let   $\mathcal{A} = \{{\bf a}\in[N]^K \ \big| \  {\bf a}(i) \neq {\bf a}(j)\ \forall \ i,j:\ 1 \leq i< j\leq K\}$   be the set of all such actions which can be constructed using $N$ arms. Thus the cardinality of $\mathcal{A}$ is $\binom{N}{K}$. We denote the action played at time $t$ as ${\bf a}_t \in \mathcal{A}$.   For an action ${\bf a}=(a_1, a_2, \cdots, a_K)$  let ${\bf d}_{{\bf a}, t} = \big(X_{a_1,t}, X_{a_2,t}, \allowbreak \dots, X_{a_K,t}\big)\in [0,1]^K$ denote the column vector of arm rewards at time $t$ from action ${\bf a}$. 
The reward $r_{\bf a}(t)$ of  action ${\bf a}$ at time $t$ is a bounded function $f:[0,1]^K\to[0,1]$ of the rewards from  the arms chosen in that action, 
\begin{align*}
r_{\bf a}(t) &= f({\bf d}_{{\bf a}, t}). \label{eq:R_at_t}
\end{align*}

Later in the text, we will skip index $t$ for brevity, where it is unambiguous. If at time $t$, action ${\bf a}_t$ is played, ${\bf d}_{{\bf a}_t, t}$ will be simplified to ${\bf d}_{{\bf a}_t}$. Also if we analyze behavior for action ${\bf a}$ such that its reward vector is independent and identically distributed across time, we will drop the subscript and use only ${\bf d}_{\bf a}$ for brevity.

 In many practical systems, the real reward is a non-linear function of noisy reward instead of a non-linear function of expected rewards plus noise. For example, consider a distributed system. A job may be forked into multiple parallel tasks. The completion time of the job depends is the maximum time taken to finish any of the task. Hence, the completion time of the job is a non-linear function of the completion time of the sub-tasks. Another example for non-linear function of individual rewards is item selection in cross selling. In cross-selling the total profit of the seller is sum of individual profits plus an additional advantage made from the combined transaction. The additional advantage can be modelled as a quadratic term of the individual items sold \citep{1250942}.

For a linear case, the two cases are equivalent as

\begin{align}
    r_t = f({\bf d}_{{\bf a}_t}) = \theta^T{\bf d}_{{\bf a}_t} = \theta^T(\mathbb{E}\left[{\bf d}_{{\bf a}_t}\right]+\eta_t) = \theta^T\mathbb{E}\left[{\bf d}_{{\bf a}_t}\right]+\theta^T\eta_t  = f\left(\mathbb{E}\left[{\bf d}_{{\bf a}_t}\right]\right) + \epsilon_t
\end{align}

For non-linear case, this formulation does not simply reduce to function of expected rewards plus noise. However, we can still write the reward as the expected bandit reward plus noise, or
\begin{align}
    r_t &= \mathbb{E}[f({\bf d}_{{\bf a}_t})] + \epsilon_t
\end{align}
Similarly to the work of \citep{auer2010ucb}, we aim to maximize the expected reward for the actions selected over time. Further, we intend to improve the finite time regret bounds for the same. We denote the expected reward of any action ${\bf a} \in \mathcal{A}$ as $\mu_{\bf a}$, or $\mu_{\bf a} = \mathbb{E}_{{\bf d}_{{\bf a}_t}|{\bf a}_t={\bf a}}[r_{\bf a}]$.

We assume that there is an action ${\bf a}^*$ for which the expected reward  $\mu_{ {\bf a}^*}$ is highest among all actions ${\bf a} \in \mathcal{A}$,
\begin{align*}
{\bf a}^* &= \arg\max_{{\bf a}\in\mathcal{A}}\mu_{\bf a}
\end{align*} 
We refer to this action ${\bf a}^*$ as ``optimal.''
%
Given an optimal action, regret for an action at time $t$ can be defined as follows. 

\begin{definition}[Regret]
    The regret of an action ${\bf a}_t$ at time $t$ is defined as the difference between the reward obtained by the optimal action and the reward obtained by ${\bf a}_t$, or 
    \begin{eqnarray}
        R(t) = r_{{\bf a}^*}(t) - r_{{\bf a}_t}(t)
    \end{eqnarray}
\end{definition}

The  objective is to minimize the expected regret {(also known as pseudo-regret)} accumulated during the entire time horizon, 
\begin{align}
W(T)= \mathbb{E}_{{\bf a}_1, r_{{\bf a}_1}(1), \cdots, {\bf a}_T, r_{{\bf a}_T}(T) }\left[\sum^{T}_{t=1}R(t)\right]
\end{align}


\subsection{Assumptions}\label{appendix_formulate}
We now discuss the technical assumptions which are required to prove the regret bounds for the \NAM\ algorithm. We note that many of the assumptions are not required for the algorithm to work, only to prove the guarantees. 

The function $f$ is assumed to be symmetric, so that the ordering within the tuple   does not matter. In other words, the rewards for an action is symmetric in its constituent arms. This assumption is true for certain problem settings where the ordering among the individual arms is not important, like the maximum of rewards, or the sum of rewards of the individual arms. This assumption is given as follows.  
\begin{assumption}[Symmetry] \label{symmentry_assumption}
     $f$ is a symmetric function of the rewards obtained by the constituent arms. More precisely, let $\Pi({\bf d})$, be an arbitrary permutation of ${\bf d}$. Then, the reward observed will be identical for both $\Pi({\bf d})$ and ${\bf d}$, or 
\begin{eqnarray}
    f\left({\bf d}\right) &=& f\left(\Pi\left({\bf d}\right)\right)
\end{eqnarray}
\end{assumption}

In the rest of the text, we denote $\Pi(\cdot)$ as a permutation, where $\Pi({\bf y})$ is one of the possible permutations {of the vector {\bf y}}. We now define gap $\Delta$ between two actions as follows.
\begin{definition}[Gap]
    The Gap $\Delta_{{\bf a}_1, {\bf a}_2}$ between  any two actions ${\bf a}_1, {\bf a}_2 \in \mathcal{A}$ is defined as the difference of expected rewards of the actions, or 
    \begin{eqnarray}
        \Delta_{{\bf a}_1, {\bf a}_2} = \mu_{{\bf a}_1} - \mu_{{\bf a}_2} = \mathbb{E}[f({\bf d}_{{\bf a}_1})] - \mathbb{E}[f({\bf d}_{{\bf a}_2})]
    \end{eqnarray}
\end{definition}

We assume that there is an optimal action ${\bf a}^*$ for which the expected reward is highest among all actions ${\bf a} \in \mathcal{A}$. We denote the reward of the optimal action by $\mu_{{\bf a}^*}$. The Gap of an action ${\bf a} \in \mathcal{A}$ with respect to the optimal action is simply written as $\Delta_{\bf a}$, or, $\Delta_{\bf a} = \Delta_{{\bf a}^*, {\bf a}}$

\begin{remark} \label{ert1}
Using linearity of expectation it can be seen that, $\mathbb{E}\left[R(t)\right] = \Delta_{{\bf a}_t}$
\end{remark}

From remark \ref{ert1}, the expected regret accumulated during the entire time horizon can be written as,
\begin{eqnarray}
  W(T) =  \mathbb{E}\left[\sum^{T}_{t=1}R(t)\right] =  \sum^{T}_{t=1}\Delta_{{\bf a}_t}
\end{eqnarray}
We define the maximum regret of all possible actions as $R_{max} = \max_{{\bf a}\in \mathcal{A}}{\Delta_{{\bf a}}}$.

We now use the concept of stochastic dominance to order two arms. Assume that there exists a first-order stochastic dominance between any two arms which is defined as follows. 
\begin{definition}[First-Order Stochastic Dominance]
     A random variable $X$ has first-order stochastic dominance (FSD) over another random variable $Y$ (or $X \succ Y$), if   for any outcome $x$, $X$ gives at least as high a probability of receiving at least $x$ as does $Y$, and for some $x$, $X$ gives a higher probability of receiving at least $x$, 
     \begin{align}
        X \succ Y \Leftrightarrow P\left(X \ge x\right) \geq P\left(Y \ge x\right) \forall x \in \mathbb{R}, \nonumber \\
        P\left(X \ge x\right) > P\left(Y \ge x\right) \text{ for some } x \in \mathbb{R}. \label{eq:majorization_def}
    \end{align}
\end{definition}

\begin{assumption}[FSD between arms] \label{majorized_arms}
    There exists a 
    dominance ordering between all the arms, which is defined using FSD. In other words, for each pair of arms $i$ and $j$, either $X_{i} \succ X_{j}$ or $X_{j} \succ X_{j}$.
\end{assumption}

The FSD implies second order stochastic dominance, which indicates that the mean of the dominating random variable is at least as much as the mean of the dominated random variable \citep{hadar1969rules,bawa1975optimal}. This is summarized in the following lemma. 
\begin{lemma}[\citep{hadar1969rules,bawa1975optimal}] \label{expectation_from_majorization}
    If a random variable $X$ has FSD over another random variable $Y$ 
    (or, $X \succ Y$ ), then the expected value of $X$ is at least the expected value of $Y$, or 
    \begin{eqnarray}
        \mathbb{E}\left[X\right] > \mathbb{E}\left[Y\right]
    \end{eqnarray}
\end{lemma}

\begin{remark}
From Lemma \ref{expectation_from_majorization}, we note that if arm $i$ dominates arm $j$, then the mean reward for arm $i$ is strictly greater than that of arm $j$. Thus, under Assumption~\ref{majorized_arms}, 
    \begin{align}
        \mathbb{E}[X_{i}] \neq \mathbb{E}[X_{j}], \forall i, j \in \{1,\cdots,N\} \label{eq:stric_order}
    \end{align}
\end{remark}
Such strict dominance exists for Bernoulli and exponential reward distribution functions. 

Since we can construct a new action by changing the arms of an existing action, we define the  replacement function $h(\cdot)$ which changes an element $i$ of a given reward vector ${\bf d}$ (where each entry in the reward vector is a random variable with the distribution of the corresponding arms).
\begin{definition}[Replacement function]
    The replacement function $h(\cdot)$ is defined as a function on $\mathbb{R}^{K+2}$, which replaces the $i^{th}$ element of vector ${\bf d}$ with $x$, or 
    \begin{eqnarray*}
     h({\bf d}, i, x)  = \left({\bf d}(1),\  \dots,\ {\bf d}(i-1),\ x,\ {\bf d}(i+1),\ \dots,\ {\bf d}(K) \right).
    \end{eqnarray*}
\end{definition}

For a random variable $X$, $h({\bf d}, i, X)$ is also a random variable. 
We also assume that the expected reward of an action is strictly increasing function of the rewards obtained by the individual arms.

\begin{assumption}[Strictly Increasing] \label{monotone_assumption}
    $f(\cdot)$ is element-wise, strictly increasing function of the individual rewards obtained by the constituent arms. More precisely, 
    \begin{equation}
        f\left(h\left({\bf d}, i, x\right)\right)\  > f\left(h\left({\bf d}, i, y\right)\right) 
        \  \forall x > y \text{ ; }  x, y \in [0,1]\ \forall{\bf d}\in\mathbb{R}^K 
    \end{equation}
\end{assumption}
Even though we assume  strictly increasing function,  the analysis also holds for strictly decreasing function $f$ by transforming the reward function as $f_{n}({\bf d}) = 1-f({\bf d})$. In order to compare the distance between individual reward vectors from two different actions, we need to find the difference in the two individual reward vectors up to a permutation, since the reward function is permutation invariant. With this distance metric in mind, we assume that $f(\cdot)$ is Lipschitz continuous (in an expected sense), which is formally described in the following. 
\begin{assumption}[Continuity of expected rewards] \label{continuous_assumption}
    The expected value of $f(\cdot)$ is Lipschitz continuous with respect to the expected value of the rewards obtained by the individual arms, meaning 
    \begin{align}
        \big|\ \mathbb{E}\left[f({\bf d}_1)\right] - \mathbb{E}\left[f({\bf d}_2)\right]\big| \ 
        \leq \ U_1 \min_{\Pi}{\big|\big|}\mathbb{E}[{\bf d}_1] - \Pi(\mathbb{E}[{\bf d}_2]){\big|\big|_2} \label{eq:cont_lower_defn_n}
    \end{align}
    for any given  random vectors ${\bf d}_1$ and ${\bf d}_2$ and for some $U_1<\infty$,  where $\Pi$ is minimized over all permutations of $\{1, \cdots, K\}$. 
\end{assumption}   
 

\begin{corollary} \label{single_variable_continuity}
	Assumption \ref{continuous_assumption} also implies
	\begin{align}
	    \big|\mathbb{E}\left[f(h\left({\bf d}, i, X\right))\right] - \mathbb{E}\left[f(h\left({\bf d}, i, Y\right))\right]\big| \ 
	    \leq\ U_1 \big| \mathbb{E}[X] - \mathbb{E}[Y]\big| \label{eq:cont_lower_defn}
	\end{align} 

	for any given random vector ${\bf d}$ and any $i \in \{1, \cdots, K\}$. 
\end{corollary}
We further assume a lower bound in \eqref{eq:cont_lower_defn} as formally stated in the following assumption. 
\begin{assumption}[Continuity of individual expected rewards] \label{inverse_continuity}
    We also assume that the continuity given in \eqref{eq:cont_lower_defn} also has a similar lower bound. More precisely, there is a $U_2<\infty$ such that 
    \begin{align}
    \big|\mathbb{E}[X] - \mathbb{E}[Y]\big|  
     \leq  U_2\big|\left(\mathbb{E}\left[f(h\left({\bf d}, i, X\right))\right] - \mathbb{E}\left[f(h\left({\bf d}, i, Y\right))\right]\right)\big|
    \end{align}
    for any given  random vectors ${\bf d}$ and any $i \in \{1, \cdots, K\}$. 
\end{assumption}   
Assumption \ref{inverse_continuity} holds for many well behaved functions in practical scenarios. For example, $f(\cdot) = \max(\cdot)$ with Bernoulli rewards for individual arms\footnote{We note that maximum function, in general, does not satisfy Assumption \ref{inverse_continuity}. However, if an arm exists such that its rewards are always higher compared to other arms, the agent can place any other arm among the $K-1$ arms, without incurring any regret.}, $f(\cdot)$ =  sum of individual rewards, or $f(\cdot)$ = concave utility function of sum of individual rewards.

\begin{corollary} \label{double_sided_continuity}
	Combining Corollary \ref{single_variable_continuity} and Assumption \ref{inverse_continuity} and defining $U = \max(U_1, U_2)$, we have
	\begin{align}
	    \frac{1}{U}\big|\mathbb{E}[X] - \mathbb{E}[Y]\big| 
	    &\leq \big|\left(\mathbb{E}\left[f(h\left({\bf d}, i, X\right))\right] - \mathbb{E}\left[f(h\left({\bf d}, i, Y\right))\right]\right)\big|\\
	    &\leq U \big|\mathbb{E}[X] - \mathbb{E}[Y]\big|, \label{eq:cont_defn}
	\end{align}
	for any given random vector ${\bf d}$ and any $i \in \{1, \cdots, K\}$. 
\end{corollary}

We note that linear bandits become a special case of the assumptions we considered.


\section{Proposed Algorithm} \label{algorithm_and_analysis}
The proposed algorithm, called \NAM, is an explore then exploit strategy which aims to minimize the expected regret, be computationally efficient, and have a storage complexity which is linear with $N$ and independent of $K$. \NAM, described in Algorithm \ref{alg:dama}, utilizes the fact that for CMAB problem, choosing $K$ arms from a set of $K+1$ arms has $K+1$ actions thus making the problem solvable using the standard MAB approach. In other words, if $N = K+1$, then the complexity is $\binom{K+1}{K} = K+1$, and only $K+1$ actions  needs to be optimized. 

 We construct a group $G$ which is a vector of length $K+1$ consisting of arm indices, or $G\in[N]^{K+1}$. Then, we can construct $K+1$ actions, each action using all but one entries in the group $G$. 

Let ${\bf a}^G_{-i}$ be an action in group $G$ with $G(i)^{th}$ arm left out, { where $G(i), i\in\{1, \cdots, K+1\}$, is the $i^{th}$ entry of the group}. 


With $X_{G(i)}$ denoting the reward of arm $G(i)$, the individual reward vector ${\bf d}_{{\bf a}^G_{-i}}$ with the action ${\bf a}^G_{-i}$ is
\begin{align}
{\bf d}_{{\bf a}^G_{-i}} = (X_{G(1)}, \dots, X_{G(i-1)}, X_{G(i+1)}, \dots,X_{G(K+1)}).
\end{align}

The (random) reward obtained at any time with this action is $r_{{\bf a}^G_{-i}} = f\left({\bf d}_{{\bf a}^G_{-i}}\right)$, with a mean reward of $\mu_{{\bf a}^G_{-i}} = \mathbb{E}\left[f\left({\bf d}_{{\bf a}^G_{-i}}\right)\right]$. The next result shows that an ordering on $\binom{K+1}{K}$ actions made using $K+1$ arms gives an ordering on $K+1$ arms under the considered assumptions.

\begin{lemma} \label{ordering_from_expectation}
	An ordering on $\binom{K+1}{K}$ actions, in group $G$, made using  $K+1$ arms gives an ordering on $K+1$ arms. In other words, if an ordering exists between actions ${\bf a}^G_{-i}$ and ${\bf a}^G_{-j}$, then an ordering exists between arms $G(i)$ and $G(j)$. More precisely, 
	\begin{eqnarray}
	\mu_{{\bf a}^G_{-i}} > \mu_{{\bf a}^G_{-j}}\Rightarrow \mathbb{E}\left[X_{G(i)}\right] < \mathbb{E}\left[X_{G(j)}\right] \label{eq:sort_condition}
	\end{eqnarray}
\end{lemma} 
\begin{proof}
    (Outline): If we have two actions made from group $G$ by excluding arm $G(i)$ and arm $G(j)$ respectively, then the arm with higher reward will increase the joint reward $f$ as $f$ is an increasing function. The detailed proof is provided in Appendix \ref{correctness}.
\end{proof}

\NAM\ divides all $N$ arms into groups of $K+1$ arms arbitrarily. Each group now contains only $K+1$ actions. If the last group contains less than $K+1$ arms (if $N\mod(K+1) > 0$), arms from other groups are added (repeated) to have $K+1$ arms in the last group. \NAM\ then picks the first group of $K+1$ arms and orders the arms in the group using SORT subroutine. Using this subroutine, the $K+1$ arms in the group are ordered with respect to expected individual rewards. We also consider $G^*$ as the best $K$ arms seen so far, which are the top $K$ arms in $G_1$. It later proceeds in $k \in \left\{2, \cdots, \frac{N}{K+1}\right\}$ rounds. In $k^{th}$ round it performs SORT on $G_k$ and merge $G_k$ and $G^*$ using MERGE subroutine to obtain a new $G^*$. The SORT subroutine orders the $K+1$ arms in $G_k$. The MERGE subroutine takes the best $K$ arms before this round, $G^*$; and the best $K$ arms from the SORT subroutine on $G_k$ and merges them to find the best $K$ arms seen so far and saves them as $G^*$. This is then inputted to the next value of $k$ to merge with other groups. 

At the end of $\left(\frac{N}{K+1}\right)^{th}$ round, we would have played all arms in each group and merged them, thus resulting in an optimal action which maximizes the expected reward for the remaining time slots.  Apart from the sort and merge scheme, we also use a hyperparameter $\lambda$ in our algorithm. $\lambda$ denotes the minimum gap the agent can resolve between any two arms. If any the gap between arm $i$ and arm $j$, the algorithm cannot determing which arm is better with high probability and selects the arm with higher sample mean as the better arm. This behaviour is common in both SORT and MERGE subroutine. We now describe the algorithms used in \NAM\ which are SORT and MERGE subroutines in detail. 


\subsection{SORT} \label{DA}
The SORT subroutine is given in Algorithm \ref{alg:da}. In this subroutine, we play $K+1$ actions formed from $K+1$ arms in a group $G$, each action corresponding to one left out arm. The subroutine proceeds in rounds similar to UCB algorithm by \citep{auer2010ucb}. By the end of round $r$, each action is played $n_r$ times so that the expected reward of each action can be estimated within $\pm \Delta_r$. At the end of each round, the estimates are used to sort the arms, where the arms $G(i), G(j)$ are considered sorted when the upper bound on reward estimate of action ${\bf a}^G_{-i}$ is less than lower bound of action ${\bf a}^G_{-j}$. When an arm is placed at its true sorted location in the group, its corresponding action is not sampled again. The procedure ends when $\Delta_r<\lambda$ or when all $K+1$ arms are sorted. At the end of the algorithm, only top $K$ arms are provided as output.

\subsection{MERGE} \label{MA}
The MERGE subroutine is given in Algorithm \ref{alg:ma}. The MERGE subroutine aims to merge two groups, each with $K$ sorted arms to give sorted best $K$ arms. Since we only want the best $K$ arms from the merged $2K$ arms to be sorted, it can be done with only $K+1$ arm comparisons. 


Starting with two $K$-sized sorted groups $G_1$, and $G_2$ and an optimal group which is empty at the start of the subroutine, we identify the best $K$ out of $2\times K$ arms by figuring the best arm one by one. Starting with both $i$ and $j$ as 1, we construct a new action by replacing the $i^{th}$ arm of group $G_1$ by the $j^{th}$ arm of group $G_2$. Note that if after replacement, the reward is bigger, it implies that the added $j^{th}$ arm of $G_2$ is the next arm in the sorted final list else the $i^{th}$ arm of $G_1$ is the next arm in the sorted final list. In order to differentiate between the two actions, the procedure similar to the SORT subroutine is used. Based on whether the $i^{th}$ arm of $G_1$  or $j^{th}$ arm of $G_2$ made in the optimal set, $i$ or $j$ is incremented and the procedure is repeated till the $K$ best arms in the merger of the two groups are obtained. 

\subsection{Complexity of \NAM}\label{sec:complexity}
We now analyze the complexity of \NAM\ for both storage and computation at each time step. Detailed subroutines are provided in Appendix \ref{detail_sm_algos}, with the key pseudo-codes in Algorithm \ref{alg:dama}-\ref{alg:ma}.

The algorithm while running SORT, or MERGE subroutine stores the reward of each action in the group, and sorts all the actions. The total storage at any step is no more than $O(K)$. Even when the groups are being merged,  $O(K)$ temporary storage is used for the merged rewards. This merged group is then used to decide the action in the exploiting phase.  Thus, the maximum storage at any time is $O(K)$ for the subroutines and $O(N)$ for \NAM. To evaluate the computational complexity at each time-step, we consider the three cases of what the algorithm may be doing at a time step. 

\begin{algorithm} [!htb]
	\caption{\textsc{Cmab-sm}($T, N, K$)}\label{alg:dama}
	\small
	\begin{algorithmic}[1]
		\State Initialize separation threshold for mean estimates
		\begin{align}
		    \lambda=\left(\frac{256N\log{2NT}}{T}\right)^\frac{1}{3}    \label{eq:define_lambda}
		\end{align}
		\State Partition $N$ arms arbitrarily into $\frac{N}{K+1}$ equal sized groups. $G_k:\ 1\leq k\leq \frac{N}{K+1}$

        \State $G^*$ = \Call{Sort}{$G_1, \lambda, T, N, K$} \Comment{Sort $1^{st}$ group. For a single group, returned set is optimal}
		\For{$k=2:\frac{N}{K+1}$}\Comment{Sort and Merge subsequent groups}
		    \State $G_k$ = \Call{Sort}{$G_k, \lambda, T, N, K$}
            \State $G^*$ = \Call{Merge}{$G^*, G_k, \lambda, T, N, K$}
		\EndFor
		\State \textbf{return} $G^*$
	\end{algorithmic}
\end{algorithm}
\begin{algorithm}[!htb]
    \caption{\textsc{Sort}($G, \lambda$, $T$, $N$, $K$)}\label{alg:da}
	\small
	\begin{algorithmic}[1]
		\State Initialize $\hat{\mu}_i = 0$ for $i \in \{1, 2, \cdots, K+1\}$ \Comment{$\hat{\mu}_i$ is the sample mean of ${\bf a}^G_{-i}$}
		\State $S\leftarrow\phi$ \Comment{list of sorted actions}
		\State Initialize $r\leftarrow1$; $\Delta_r\leftarrow\frac{1}{2^r}$; $n_r\leftarrow \frac{2\log{(TNK)}}{\Delta_r^2}$\Comment{Initial confidence bounds}
        \While{$\Delta_r > \lambda$ \textbf{and} $|S|<K+1$}
            \ForAll{$i \in \{1, 2, \cdots, K+1\}$}
                \If {$|\hat{\mu}_i - \hat{\mu}_j|\leq 2\Delta_r$ for some $j \in \{1, 2, \cdots, K+1\}\setminus\{i\}$ }
                    \State Sample action ${\bf a}^G_{-i}$ for total of $n_r$ times and update $\hat{\mu}_i$
                \ElsIf {$i \notin S$}
                    \State $S = S\cup\{i\}$
                \EndIf
            \EndFor
            \State Update $r=r+1$; $\Delta_r=\frac{\Delta_{r}}{2}$ ; $n_r =\frac{2\log{(TNK)}}{\Delta_r^2}$
        \EndWhile
        \State Make $S$ of length $K+1$ (if not already) by adding remaining indices from $G\setminus S$
        \State Order $S$ inverse sorting on $\{\hat{\mu}\}_{1:K+1}$
        \State \textbf{return} $S[1:K]$
	\end{algorithmic}
\end{algorithm}
\begin{algorithm}[!htb]
	\caption{\textsc{Merge}($G_1, G_2, \lambda, T, N, K$)}\label{alg:ma}
	\small
	\begin{algorithmic}[1]
		\State Initialize $G^*[1:K]\leftarrow \bf{0}$ \Comment{Array of size K initialized with $0$ to store the optimal arms}
		\State ${\bf a}_{G_1} = G_1$ \Comment{Construct action from arms in Group 1}
        \State Initialize $\hat{\mu}_{{\bf a}_{G_1}} \leftarrow 0$;  $i\leftarrow 1;\ j\leftarrow 1$
        \State Initialize $r_1\leftarrow 1$; $\Delta_{r_1}\leftarrow\frac{1}{2}$ ; $n_{r_1} =\frac{2\log{(TNK)}}{\Delta_{r_1}^2}$
	    \For{$k=1:K$}
            \State ${\bf a}_{i,j} = \left(G_1 \setminus \{G_1(i)\}\right) \cup \{G_2(j)\}$ \Comment{Replace $i^{th}$ arm of $G_1$ by $j^{th}$ arm of $G_2$}
            \State Initialize $r_2\leftarrow 1$; $\Delta_{r_2}\leftarrow\frac{1}{2}$ ; $n_{r_2} =\frac{2\log{(TNK)}}{\Delta_{r_2}^2}$; $\hat{\mu}_{{\bf a}_{i,j}} \leftarrow 0$ \Comment{For every new action constructed}
		    \While{$\left(\Delta_{r_2} > \lambda\right)$ \textbf{and} $\left(G^*[k] == 0\right)$}
                \State While total samples of action ${\bf a}_{G_1}$ < $n_{r_1}$, Sample action ${\bf a}_{G_1}$ to update $\hat{\mu}_{{\bf a}_{G_1}}$
                \State While total samples of action ${\bf a}_{i,j}$ < $n_{r_2}$, Sample action ${\bf a}_{i,j}$ to update $\hat{\mu}_{{\bf a}_{i,j}}$
                \If{confidence bounds of ${\bf a}_{G_1}$ and ${\bf a}_{i,j}$ do not overlap}
                    \State Update $G^*[k], i, j$ 
                \Else
                    \State Update $r_2=r_2+1$; $\Delta_{r_2}=\frac{\Delta_{r_2}}{2}$ ; $n_{r_2} =\frac{2\log{(TNK)}}{\Delta_{r_2}^2}$
                    \State If $r_2 > r_1$, update $r_1=r_1+1$; $\Delta_{r_1}=\frac{\Delta_{r_1}}{2}$ ; $n_{r_1} =\frac{2\log{(TNK)}}{\Delta_{r_1}^2}$
                \EndIf
            \EndWhile
            \State Update $G^*[k], i, j$ if not updated
		\EndFor
        \State \textbf{return} $G^*$\Comment{Merged set of $G_1$, and $G_2$}
	\end{algorithmic}
\end{algorithm}

1) In SORT subroutine, at the end of each iteration of the while loop, arms in the group are sorted requiring $O(K\log K)$ computations. It then loops over all the arms to place them in the correct order which takes $O(K)$ steps. Thus, the computational complexity in the worst case time-step in sort is $O(K\log K)$. 2) MERGE subroutine at any time step either run action and saves the result, or perform comparisons which are all $O(1)$ at each time. 3) After the MERGE is complete, the best action is available which is then exploited thus making the complexity in the exploit phase as $O(1)$. Thus, the overall complexity at any time is $O(K\log K)$ which happens due to sorting the actions for the removal of sub-optimal arms after every round in SORT subroutine. 

In each call to SORT, the actions are sorted with respect to their mean observed rewards. From Lemma \ref{ordering_from_expectation}, an ordering is also obtained for the corresponding arms. In MERGE subroutine, a new action is constructed from an old action by replacing exactly one arm. The ordering between the old action and the new action gives the ordering between the replaced arm and the new arm. Note that the inequality conditions work in different directions for SORT and MERGE algorithms.

\subsection{Other design options}
We note that an algorithm can be constructed by keeping the first $K-1$ arms fixed. The algorithm will now select the best arm from remaining $N-K+1$ arms. This arm will now  always be kept in first $K-1$ arms. The process is repeated until all $K$ places are filled. However, this algorithm has two issues, 1) \textbf{Higher complexity}: The algorithm will need to sort $N-K+1$ arms into place which increases the sorting complexity from $\tilde{O}(K)$ to $\tilde{O}(N)$. \textbf{More exploration steps}. Since, this algorithm will now perform exploration after placing every an arm among the first $K$ Group. This increases the time required for exploration by a factor of $K$, which would increase the order of regret bound. This makes the \NAM\ a better choice compared to a na\"ive implementation of UCB by fixing $K-1$ arms.
\section{Regret Analysis of the proposed algorithm} \label{main_results}
In this section, we will present the  main result of the paper, related to the regret analysis of the proposed algorithm. 
\subsection{Bounds on exploitation regret and exploration time}\label{exp_regret_exp_time}

In this subsection, we will bound the regret in the exploitation phase, which indicates the loss in reward due to choosing an incorrect action at the end of the MERGE algorithm. We will also bound the time spent in the SORT and the MERGE subroutines, which is the exploration phase. 

We first find the time taken in the sort and merge subroutines. 
\begin{lemma}[Sort time requirement] \label{sort_requirements}
	SORT subroutine 2 gives correct ordering on $K+1$ actions in a group $G$ with probability $1-\frac{K}{2N^2T^2}$, {up to precision $\lambda$ defined in equation (\ref{eq:define_lambda})}, where the actions are chosen for at most 
	\begin{equation*}
	\left(\sum_{i=1}^{K+1} \frac{64 U^2 \log 2NT}{\max(\delta_{G(i)}^2, \lambda^2)}\right)
	\end{equation*}
	time steps, where $\delta_{G(i)} $ is given in \eqref{eq:delta_i_def}.
	
	\begin{figure*}[!t]
		\normalsize
		\begin{eqnarray}
		\delta_{G(i)} = \left\{\begin{array}{cr}\mathbb{E}\left[X_{G(i)}\right] - \mathbb{E}\left[X_{G(i+1)}\right], &i = 1 \\
		\min\left\{\mathbb{E}\left[X_{G(i-1)}\right] - \mathbb{E}\left[X_{G(i)}\right], \mathbb{E}\left[X_{G(i)}\right] - \mathbb{E}\left[X_{G(i+1)}\right]\right\}, &i \in {2,...,K}\\
		\mathbb{E}\left[X_{G(i-1)}\right] - \mathbb{E}\left[X_{G(i)}\right], &i = K+1 \end{array}\right. \label{eq:delta_i_def}
		\end{eqnarray}
		\hrulefill
		\vspace*{4pt}
	\end{figure*}
	
\end{lemma} 
\begin{proof}
    (Outline) We sample each action till the confidence intervals around estimated means of any two actions are separated. The confidence intervals reduces as number of samples increases from concentration bounds from Lemma \ref{Hoeffding_lemma} with a lower limit of $\lambda$. Using Corollary \ref{double_sided_continuity} and Hoeffding's Inequality (Lemma \ref{Hoeffding_lemma}) we bound the number of samples required for separation with high confidence. We then use union bounds to bound the total numbers of samples required for each action.
	The detailed proof is provided in Appendix \ref{sort_merge_lemma}.
\end{proof}
\begin{lemma}[Merge time requirement] \label{Merge_requirements}
	MERGE subroutine 5 merges arms in two groups $G_1$ and $G_2$ to $G$ correctly with probability $1-\frac{K}{2N^2T^2}$, {up to precision $\lambda$ as defined in equation (\ref{eq:define_lambda})}, where the total number of time steps needed to merge is at most 
	\begin{eqnarray*}
		\left(\sum_{i=1}^{K+1} \frac{64 U^2 \log 2NT}{\max(\delta_{G(i)}^2, \lambda^2)} \right),
	\end{eqnarray*}
where 	$\delta_{G(i)} $ is given in \eqref{eq:delta_i_def}.
\end{lemma} 
\begin{proof}
    (Outline) While merging two groups $G_1$ and $G_2$, we sample two actions till the confidence intervals around the estimated means of both actions are separated by twice the confidence intervals around the estimates. Again using Lemma \ref{Hoeffding_lemma} and Corollary \ref{double_sided_continuity}, we bound the number of samples required reduce the confidence intervals sufficiently enough to order two arms.
	The detailed proof is provided in Appendix \ref{sort_merge_lemma}.
\end{proof}

In order to bound the regret in the exploitation phase, we first characterize the probability that the action decided by \NAM\ is not the best action.  

\begin{lemma}[Total error probability] \label{algorithm_error_event_prob}
The probability that the action selected by \NAM\ during the exploitation phase is not the best action  (up to precision defined in Equation \ref{eq:define_lambda}) is at most $\frac{1}{NT^2}$. 
\end{lemma} 

\begin{proof}
    (Outline) We use union bounds to calculate total error probability of the algorithm. We use Lemma \ref{sort_requirements} and Lemma \ref{Merge_requirements} to calculate failure probability of each sort and merge event respectively. There are total $\frac{N}{K+1}$ groups to be sorted, and $\frac{N}{K+1}-1$ groups to be merged. Taking union bound over the total number of failure events, and probability of each failure event gives an upper bound on total error probability of the algorithm.
	The detailed proof is provided in Appendix \ref{proof_algorithm_error_event_prob}. 
\end{proof}

In the next result, we bound the time spent in exploring, including the SORT and MERGE subroutines for all groups. 

\begin{lemma}[Bound on exploration time steps] \label{exploration_time_bound}
Total time-steps used to SORT all $\frac{N}{K+1}$ groups, and merge these sorted groups one after the other is bounded as 
\begin{eqnarray}
    T_{exp} \le \frac{128NU^2\log{2NT}}{\lambda^2}
\end{eqnarray}
\end{lemma}
\begin{proof}
    (Outline) Lemma \ref{sort_requirements} gives the maximum number of samples required to Sort one group. Similarly,  Lemma \ref{Merge_requirements} gives the number of samples required to merge two groups. Since there are $\frac{N}{K+1}$ groups to be sorted, and $\frac{N}{K+1}-1$ groups to be merged. Summing over total number of samples for each groups gives an upper bound on total samples required for exploration.
	The detailed proof is provided in Appendix \ref{proof_exploration_time_bound}.
\end{proof}

In the following result, we bound the expected regret in the exploitation phase, caused by \NAM\ selecting incorrect action. 

\begin{lemma}[ Bounded exploitation regret] \label{exploitation_regret_bound}
The expected regret when a sub-optimal action ${\bf \hat{a}}^*$ is returned by \NAM\ is bounded as 
\begin{eqnarray}
    \mathbb{E}\left[\Delta_{{\bf \hat{a}}^*}\right] \leq U\lambda\sqrt{K} + \frac{U\sqrt{K}}{N^2T^2}
\end{eqnarray}
\end{lemma} 
\begin{proof}
    (Outline) Regret can arise in Exploitation phase when either SORT algorithm or MERGE algorithm had a failure event. Regret can also come in exploitation phase if two arms are have expected rewards close enough that SORT or MERGE algorithm cannot distinguish between them with high confidence. Combining these two sources of regret and using Assumption \ref{continuous_assumption} gives the upper bound on the expected regret during the exploitation phase.
	The detailed proof is provided in Appendix \ref{proof_exploitation_regret_bound}. 
\end{proof}

\subsection{Main Result}
Our main result is presented in Theorem \ref{final_theorem_for_regret_bound}, which states that \NAM\ algorithm achieves a sub-linear expected regret. 

\begin{theorem} \label{final_theorem_for_regret_bound}
\NAM\ algorithm described in Algorithm \ref{alg:dama} has an expected regret accumulated during the entire time horizon upper bounded as
	\begin{equation}
	W(T) = \tilde{O}\left(N^\frac{1}{3}K^\frac{1}{2}T^\frac{2}{3}\right)
	\end{equation}
\end{theorem}
\begin{proof}
(Outline) We first note that regret of the algorithm for playing sub-optimal action can come during the exploration phase, or during exploitation phase if the exploration resulted in a suboptimal action. Time steps \NAM\ uses for exploration is the total time spend in SORT and MERGE subroutines. We bound the time steps in both subroutines by $\frac{128NU^2\log{2NT}}{\lambda^2}$ using Lemma \ref{exploration_time_bound}. Further, if the algorithm results in an sub-optimal arm, the expected regret in a single time step of exploitation phase is $\left(U\lambda\sqrt{K} + \frac{U\sqrt{K}}{NT^2}\right)$ by using Lemma \ref{exploitation_regret_bound}. By choosing an optimal value of $\lambda$ as defined in \eqref{eq:define_lambda} we obtain the required bound. Having described the outline, we next give the detailed steps of the proof. 

(Detailed Steps) We note that the expected regret till time $T$ is sum of expected regret accumulated at each round. We can rewrite it as sum of two phases of the algorithm which are exploration and exploitation as,
\begin{eqnarray}
W(T)&=& \sum^{T}_{t = 1}\mathbb{E}\left[R(t)\right]\\
&=& \sum^{T_{exp}}_{t = 1}\mathbb{E}\left[R(t)\right] + \left(T-T_{exp}\right)\times\mathbb{E}\left[\Delta_{\hat{{\bf a}}^*}\right] \label{eq:max_regret_explore_exploit}\\
&\le & \sum^{T_{exp}}_{t = 1}\mathbb{E}\left[R(t)\right] + T\mathbb{E}\left[\Delta_{\hat{{\bf a}}^*}\right] \label{eq:max_regret_exploit_time}\\
&\leq& \sum^{T_{exp}}_{t = 1}\max\left(R(t)\right) + T\mathbb{E}\left[\Delta_{\hat{{\bf a}}^*}\right] \label{eq:max_regret_max_explore}\\
&\leq& T_{exp}\max\left(R(t)\right) + T\mathbb{E}\left[\Delta_{\hat{{\bf a}}^*}\right], \label{eq:max_explore_regret} 
\end{eqnarray}
where  (\ref{eq:max_regret_explore_exploit}) follows from splitting the regret into exploration-exploitation phase,  (\ref{eq:max_regret_exploit_time}) follows since $T-T_{exp}\le T$, and  (\ref{eq:max_regret_max_explore}) follows since mean is at most the maximum. 

Using the values for maximum regret in any round, Lemma \ref{exploitation_regret_bound}, inequality (\ref{eq:pbound}), maximum exploration time from Lemma \ref{exploration_time_bound}, and maximum exploitation regret from Lemma \ref{exploitation_regret_bound}, we have,
\begin{eqnarray}
&&W(T)\nonumber\\&\leq& U\sqrt{K}\frac{128NU^2\log{2NT}}{\lambda^2} + T\left(U\lambda\sqrt{K} + \frac{U\sqrt{K}}{NT^2}\right)\nonumber \\
&=& \left(\frac{128NU^3\sqrt{K}\log{2NT}}{\lambda^2} + TU\lambda\sqrt{K}\right) + \frac{U\sqrt{K}}{NT}. \label{eq:cum_reg_in_lambda}
\end{eqnarray}
We now choose a value of $\lambda$ to optimize $W(T)$. Since during the implementation of algorithm $U$ is most likely an unknown quantity, we use the following value of $\lambda$ 
\begin{equation}
\lambda=\left(\frac{256N\log{2NT}}{T}\right)^\frac{1}{3}    \label{eq:proof_define_lambda}
\end{equation}
Choosing the value of $\lambda$ as defined in  (\ref{eq:proof_define_lambda}), we have the total regret of the algorithm as
\begin{eqnarray*}
	W(T)&\leq& 
	(U^3+2U)\sqrt{K}\left(32N\log{2NT}\right)^{\frac{1}{3}}T^{\frac{2}{3}} + \frac{U\sqrt{K}}{NT}
\end{eqnarray*}
This proves the result as in the statement of the Theorem. 
\end{proof}

{This trick where we tune $\lambda$ after we define the precision in each SORT/MERGE round allows us to eliminate the dependence on potentially hard to order sequences of items. The intuition behind this is, in a finite time horizon, any agent wants to work out the best possible it can get however can do so only up to a certain precision permitted by finite the time available.}

\subsection{Handling insufficient exploration time}
We note that there can be a instance where the algorithm is run with insufficient time for exploration. We first characterize what value of $T$ would result in the algorithm to run with insufficient exploration time. Then, we evaluate the regret in such a scenario indeed occurs.

Note that we assumed that rewards of each arm lies between $[0,1]$ in Section \ref{formulation}. This results in gap between any two arms is less than 1. For the optimal $\lambda$ defined in Equation \eqref{eq:define_lambda}, any $T \leq 256N\log(2NT)$ will make $\lambda \geq 1$ which serves no practical purpose based on our assumption. In that case, we arbitrarily select one of the ${{N}\choose{K}}$ actions and still suffer a maximum regret of $TU\sqrt{K} \leq \lambda TU\sqrt{K}$. We have,
\begin{align}
    W(T) &\leq U\sqrt{K}T \label{eq:linear_regret}\\
        &\leq TU\sqrt{K}\lambda\\
        &= U\sqrt{K}(256N\log(2NT))^{1/3}T^{2/3}\label{eq:high_constant_sub_linear}
\end{align}
We note that the sub-linear regret in Equation \ref{eq:high_constant_sub_linear} grows as $\tilde{O}(T^{2/3})$ but with a large multiplicative constant. Hence the linear regret in Equation \ref{eq:linear_regret} provides a better bound because of limited time to explore all arms.

\subsection{Handling unknown time horizon using doubling trick}
We now analyze the case where the time horizon $T$ is unknown and the algorithm requires to optimize actions without the knowledge of $T$ to tune $\lambda$. We use the standard doubling trick from Multi-Armed Bandit literature \citep{auer2010ucb,besson2018doubling}. To use doubling trick we start the algorithm from $T_0 = 0$. We the restart the algorithm after every $T_l = 2^l,\ l=1,2, \cdots$ time steps, till the algorithm reaches the unknown $T$. Each restart of the algorithm runs for $T_l - T_{l-1}$ steps with $T_0 = 0$ with $\lambda_l = \left(\frac{256 N\log 2N(T_l - T_{l-1})}{T_l - T_{l-1}}\right)^{1/3}$

To show that the regret is bounded by $T^{2/3}$ for the doubling algorithm, we use Theorem 4 from \citep{besson2018doubling} which we state in the following lemma.

\begin{lemma}\label{doubling_trick_lemma}
If an algorithm $\mathcal{A}$ satisfies $R_T(\mathcal{A}_T) \leq cT^\gamma(\log T)^\delta + f(T)$, for $0< \gamma < 1$, $\delta\geq 0$ and for $c> 0$, and an increasing function $f(t) = o(t^\gamma(\log t)^\delta(\text{at } t\to\infty)$, then anytime version $\mathcal{A}' := \mathcal{DT}(\mathcal{A}, (T_i)_{i\in\mathbb{N}})$ with geometric sequence $(T_i)_{i\in\mathbb{N}}$ of parameters $T_0\in\mathbb{N}^*$, $b>1, (i.e.,T_i = \lfloor T_0b^i\rfloor)$ with the condition $T_0(b-1) > 1$ if $\delta > 0$ satisfies,
\begin{align}
    R_T(\mathcal{A}') \leq l(\gamma, \delta, T_0, b)cT^\gamma(\log T)^\delta + g(T),
\end{align}
with a increasing function $g(t) = o(t^\gamma(\log t)^\delta)$ and a constant loss $l(\gamma, \delta, T_0, b)>1$,
\begin{align}
    l(\gamma, \delta, T_0, b) := \left(\left(\frac{\log (T_0(b-1)+1)}{\log (T_0(b-1))}\right)^\delta\right)\times\frac{b^\gamma(b-1)^\gamma}{b^\gamma-1}
\end{align}
\end{lemma}

Using Lemma \ref{doubling_trick_lemma} for $b = 2, \gamma = 2/3, \delta = 1/3$, we can convert our algorithm to an anytime algorithm.


\subsection{Comparison between CMAB-SM and UCB algorithm}\label{disc}
We now compare the regret bound with the one that would be achieved by using the UCB approach on each of the $\binom{N}{K}$ actions. 

\begin{lemma}[UCB Regret, \citep{auer2010ucb}] 
	For a Multi Armed Bandit setting with action space $\mathcal{A}$, time horizon T, and precision $\lambda \approx \sqrt{\frac{|\mathcal{A}|\log{|\mathcal{A}|}}{T}}$, the expected regret accumulated during entire time horizon T using improved UCB algorithm is upper bounded as,
	\begin{equation*}
	W(T) \leq \sqrt{|\mathcal{A}|T}\frac{\log{\left(|\mathcal{A}|\log{|\mathcal{A}|}\right)}}{\sqrt{\log{|\mathcal{A}|}}}
	\end{equation*}
\end{lemma}

Bounding the size of action space by using Stirling's approximation \citep{slomson1997introduction}, we get expected regret accumulated regret at time $T$ of UCB algorithm as,
\begin{eqnarray*}
	W_{UCB}(T) = \tilde{O}\left(\left(\frac{eN}{K}\right)^{\frac{K}{2}}T^{\frac{1}{2}}\right)
\end{eqnarray*}

For UCB approach to outperform \NAM, $T$ has to be very large. More formally, 
\begin{eqnarray*}
	W(T) > W_{UCB}(T),
	\text{ when }T = \Tilde{\Omega}\left(\frac{e^{3K}N^{3K-2}}{K^{3K+3}}\right)
\end{eqnarray*}

Even for an agent which can play $10^{12}$ actions per second, this will take about $10$ million years to outperform \NAM\ for a setup with $N=30$ and $K=15$. Hence, for all practical problems when an agent has a large number of arms to play simultaneously, the \NAM\ algorithm will outperform the UCB algorithm. 

\subsection{Discussion of $\tilde{O}(T^{2/3})$ regret bound}
 Lower bound of $\Omega(\sqrt{NT})$ for Linear Bandits was proven in \citep{dani08stochastic}, and lower bound of $\Omega(\sqrt{KNT})$ is proven for semi-bandits by \citep{kveton2015tight}. Note that any $\tilde{O}(\sqrt{T})$ regret bound algorithm compares all the actions with best possible action either by getting individual regret for semi-bandits or by estimating individual regret as in linear bandits \citep{dani08stochastic,abbasi-yadkori11improved,kveton2015tight}.

Individual sub-optimal action ${\bf a}$ is eliminated in $O(\frac{1}{\Delta_{\bf a}^2})$ number of samples. The regret from this action then becomes $\Delta_{\bf a}\times \frac{1}{\Delta_{\bf a}^2} = \frac{1}{\Delta_{\bf a}} $ and hence, the cumulative regret is of the form of $\frac{1}{\lambda}T + \lambda T$. This gives $O(\sqrt{T})$ regret for $\lambda = \frac{1}{\sqrt{T}}$.

Due to the bandit feedback, we do not directly obtain the rewards of the individual arms but we only get an ordering on any two arms which can be compared. To eliminate arms early we need some estimator of individual arms similar to linear bandits. Also, the regret accumulated to eliminate arm $i$ is not of the form $O(1/\Delta_i)$. This also hinders the development of an $\tilde{O}(\sqrt{T})$ bound with linear space and time complexity.




\section{Numerical Evaluation}\label{sec:synth_eval}

In this section, we evaluate \NAM\ under a synthetic problem setting. We compare the result with improved UCB algorithm as described in \citep{auer2010ucb}. Since this paper provides the first result with non-linear reward functions for CMAB problem with bandit feedback, we compare with the UCB algorithm \citep{auer2010ucb} which is optimal for small $N$ and $K$ while having the regret scale with $\binom{N}{K}$. 

For evaluations, we ran the algorithm for $T = 10^6$ time steps and averaged over $30$ runs. We compare cumulative regret at each $t$ starting from $t=0$, which is defined as,
\begin{eqnarray*}
	W(t) = \sum^t_{t'=0}R(t')
\end{eqnarray*}
\begin{figure*}
	\hspace{-0.4cm}
	\subfigure[$N=24$]{
		\input{convex_N_24}
		\label{fig:cvx_N_24}}
	\subfigure[$N=12$]{
\begin{tikzpicture}[thick,scale=0.85, every node/.style={scale=0.85}]

\definecolor{color0}{rgb}{0.75,0,0.75}

\begin{axis}[
grid=both,
grid style={white!75!black},
axis background/.style={fill=white},
axis line style={black},
x label style={at={(axis description cs:0.5,-0.025)},anchor=north},
y label style={at={(axis description cs:-0.08,.5)},anchor=south},
xlabel={t},
ylabel={W(t)},
ymode=log,
legend cell align={left},
legend entries={{UCB, K=2},{CMAB\_SM, K=2},{UCB, K=3},{CMAB\_SM, K=3},{UCB, K=5},{CMAB\_SM, K=5}},
legend style={at={(0.53,0.35)},nodes={scale=0.6, transform shape}, anchor=north west, draw=white!80.0!black, fill=white!99.80392156862746!black},
tick pos=left,
xmajorgrids,
xmin=-49000, xmax=1029000,
ymajorgrids,
ymin=500, ymax=77665.399
]
\addlegendimage{mark=asterisk, dashed, green!50.0!black}
\addlegendimage{mark=asterisk, green!50.0!black}
\addlegendimage{mark=o, dashed, color0}
\addlegendimage{mark=o, color0}
\addlegendimage{mark=star, dashed, black}
\addlegendimage{mark=star, black}
\addplot [thick, mark=asterisk, mark repeat=10, green!50.0!black, dashed]
table [row sep=\\]{%
0	0.26 \\
20000	2720.33 \\
40000	5253.6 \\
60000	6645.73 \\
80000	7685.52 \\
100000	8326.52 \\
120000	8756.15 \\
140000	9113.96 \\
160000	9390.56 \\
180000	9601.95 \\
200000	9768.89 \\
220000	9892.8 \\
240000	9978.27 \\
260000	10049.08 \\
280000	10101.64 \\
300000	10159.28 \\
320000	10197.4 \\
340000	10248.42 \\
360000	10290.18 \\
380000	10326.29 \\
400000	10354.77 \\
420000	10374.9 \\
440000	10378.53 \\
460000	10394.76 \\
480000	10402.18 \\
500000	10396.16 \\
520000	10403.75 \\
540000	10423.7 \\
560000	10428.03 \\
580000	10432.18 \\
600000	10430.93 \\
620000	10441.65 \\
640000	10455.24 \\
660000	10469.91 \\
680000	10484.94 \\
700000	10485.41 \\
720000	10491.94 \\
740000	10500.56 \\
760000	10502.66 \\
780000	10496.13 \\
800000	10498.64 \\
820000	10487.08 \\
840000	10484.51 \\
860000	10474.74 \\
880000	10479.66 \\
900000	10489.03 \\
920000	10473.64 \\
940000	10477.86 \\
960000	10486.07 \\
980000	10478.41 \\
};
\addplot [thick, mark=asterisk, mark repeat=10, green!50.0!black]
table [row sep=\\]{%
0	0.25 \\
20000	3036.31 \\
40000	6079.41 \\
60000	9149.82 \\
80000	12148.32 \\
100000	15156.6 \\
120000	18145.41 \\
140000	21119.84 \\
160000	23982.49 \\
180000	26666.28 \\
200000	28934.86 \\
220000	31257.54 \\
240000	33418.71 \\
260000	35596.82 \\
280000	37703.63 \\
300000	39843.15 \\
320000	41835.42 \\
340000	43702.72 \\
360000	45567.23 \\
380000	47390.26 \\
400000	49045.35 \\
420000	50701.6 \\
440000	52214.66 \\
460000	53748.08 \\
480000	55246.24 \\
500000	56649.57 \\
520000	57904.59 \\
540000	59125.4 \\
560000	60249.01 \\
580000	61307.06 \\
600000	62206.64 \\
620000	63166 \\
640000	64114.1 \\
660000	65091.87 \\
680000	66049.07 \\
700000	66937.31 \\
720000	67816.65 \\
740000	68630.49 \\
760000	69393.45 \\
780000	70034.67 \\
800000	70446.03 \\
820000	70903.62 \\
840000	71369.35 \\
860000	71800.87 \\
880000	72166.6 \\
900000	72557.94 \\
920000	72954.86 \\
940000	73316.22 \\
960000	73654.25 \\
980000	73967.05 \\
};
\addplot [thick, mark=o, mark repeat=10, color0, dashed]
table [row sep=\\]{%
0	0.13 \\
20000	1328.32 \\
40000	2652.45 \\
60000	3980.85 \\
80000	5308.35 \\
100000	6633.03 \\
120000	7958.25 \\
140000	9284.13 \\
160000	10593.02 \\
180000	11899.03 \\
200000	13207.62 \\
220000	14510.5 \\
240000	15816.55 \\
260000	17119.94 \\
280000	18426.82 \\
300000	19733.57 \\
320000	21039.49 \\
340000	22346.95 \\
360000	23658.52 \\
380000	24964.08 \\
400000	26271.57 \\
420000	27574.82 \\
440000	28851.08 \\
460000	30090.74 \\
480000	30996.88 \\
500000	31812.88 \\
520000	32627.09 \\
540000	33447.44 \\
560000	34260.35 \\
580000	35073.82 \\
600000	35888.22 \\
620000	36697.18 \\
640000	37494.63 \\
660000	38289.01 \\
680000	39083.07 \\
700000	39883.09 \\
720000	40659.87 \\
740000	41439.7 \\
760000	42215.91 \\
780000	42977.52 \\
800000	43691.91 \\
820000	44374.55 \\
840000	45017.41 \\
860000	45647.59 \\
880000	46237.97 \\
900000	46821.95 \\
920000	47380.1 \\
940000	47916.96 \\
960000	48420.75 \\
980000	48898.79 \\
};
\addplot [thick, mark=o, mark repeat=10, color0]
table [row sep=\\]{%
0	0.16 \\
20000	1344.77 \\
40000	2600.1 \\
60000	3942.74 \\
80000	5028.24 \\
100000	5921.01 \\
120000	6800.42 \\
140000	7620.48 \\
160000	8322.89 \\
180000	8579.49 \\
200000	8924.68 \\
220000	9261.51 \\
240000	9573.02 \\
260000	9692.4 \\
280000	9692.75 \\
300000	9694.07 \\
320000	9698.49 \\
340000	9703.89 \\
360000	9701.23 \\
380000	9699.83 \\
400000	9695.06 \\
420000	9692.84 \\
440000	9693.88 \\
460000	9697.62 \\
480000	9697.77 \\
500000	9700.66 \\
520000	9695.6 \\
540000	9694.43 \\
560000	9699.99 \\
580000	9702.68 \\
600000	9694.7 \\
620000	9693.14 \\
640000	9698.36 \\
660000	9701.92 \\
680000	9694.08 \\
700000	9692.58 \\
720000	9695.61 \\
740000	9690.83 \\
760000	9696.31 \\
780000	9696.44 \\
800000	9694.93 \\
820000	9692.57 \\
840000	9693.56 \\
860000	9694.29 \\
880000	9693.5 \\
900000	9691.5 \\
920000	9697.5 \\
940000	9702.92 \\
960000	9705.66 \\
980000	9705.13 \\
};
\addplot [thick, mark=star, mark repeat=10, black, dashed]
table [row sep=\\]{%
0	0.07 \\
20000	729.09 \\
40000	1457.97 \\
60000	2186.26 \\
80000	2917.58 \\
100000	3646.89 \\
120000	4378.36 \\
140000	5105.9 \\
160000	5834.69 \\
180000	6565.41 \\
200000	7295.68 \\
220000	8027.04 \\
240000	8756.14 \\
260000	9486.76 \\
280000	10216.9 \\
300000	10944.6 \\
320000	11676.25 \\
340000	12405.78 \\
360000	13135.62 \\
380000	13865.19 \\
400000	14594.51 \\
420000	15325.78 \\
440000	16055.16 \\
460000	16784.46 \\
480000	17515.29 \\
500000	18242.54 \\
520000	18973.48 \\
540000	19704.24 \\
560000	20432.56 \\
580000	21161.44 \\
600000	21890.27 \\
620000	22622.01 \\
640000	23350.61 \\
660000	24081.34 \\
680000	24810.62 \\
700000	25543.38 \\
720000	26273.46 \\
740000	27004.96 \\
760000	27732.83 \\
780000	28464.57 \\
800000	29196.27 \\
820000	29925.37 \\
840000	30655.71 \\
860000	31386.17 \\
880000	32113.81 \\
900000	32842.73 \\
920000	33572.82 \\
940000	34300.78 \\
960000	35030.07 \\
980000	35760.73 \\
};
\addplot [thick, mark=star, mark repeat=10, black]
table [row sep=\\]{%
0	0.08 \\
20000	725.01 \\
40000	1452.11 \\
60000	2185.74 \\
80000	2789.68 \\
100000	3288.19 \\
120000	3775.46 \\
140000	4270.07 \\
160000	4768.02 \\
180000	5288.82 \\
200000	5807.99 \\
220000	6007.14 \\
240000	6012.95 \\
260000	6018.76 \\
280000	6021.88 \\
300000	6029.86 \\
320000	6037.72 \\
340000	6044.05 \\
360000	6048.83 \\
380000	6052.94 \\
400000	6058.12 \\
420000	6063.44 \\
440000	6066.34 \\
460000	6067.15 \\
480000	6071.14 \\
500000	6075.93 \\
520000	6078.84 \\
540000	6086.24 \\
560000	6092.99 \\
580000	6097.09 \\
600000	6100.71 \\
620000	6105.53 \\
640000	6111.42 \\
660000	6117.11 \\
680000	6120.51 \\
700000	6125.06 \\
720000	6131.93 \\
740000	6138.67 \\
760000	6144.76 \\
780000	6151.38 \\
800000	6158.84 \\
820000	6164.39 \\
840000	6167.39 \\
860000	6171.55 \\
880000	6174.58 \\
900000	6182.48 \\
920000	6190.57 \\
940000	6196.7 \\
960000	6205.42 \\
980000	6208.12 \\
};
\end{axis}

\end{tikzpicture}
		\label{fig:cvx_N_12}}
	\caption{Empirical regret of UCB and CMAB-SM algorithm for the case when the reward of actions is a non linear function of rewards of individual arms for various values of $N$ and $K$. As it can be seen from the plots, except for the case of $K=2$, CMAB-SM incurs significantly lower regret than UCB algorithm.}\label{fig:cvx_of_rewards}
\end{figure*}
We consider two values of $N \in \{12, 24\}$. For $N=12$, we choose $K \in \{2, 3, 5\}$, while for $N=24$, we choose $K \in \{2, 3, 5, 7, 11\}$. Since the arms must have FSD over each other, we describe one example single parameter distributions for the reward of each arm that have this property. We consider a random variable $Y_i$ which follows an exponential distribution with parameter $\lambda_i$,  $Y_i \sim exp(\lambda_i)$. Since this random variable can take values in $[0,\infty)$, we transform the variable using $\arctan$ function to limit it to the set $[0,\pi/2)$ as 
\begin{eqnarray}
X_i &=& \frac{2}{\pi}\arctan(Y_i). \label{eq:ctfn}
\end{eqnarray}
We note that arm $i$ has FSD over arm $j$ if $\lambda_i>\lambda_j$.  Thus, this reward distribution satisfies Assumption \ref{majorized_arms} as long as no two arms have same parameter, or $\lambda_i\ne \lambda_j$ for any $i\ne j$. Figure \ref{fig:exponential} in Appendix plots $P(X\ge x)$ of the reward function for different values of $\lambda_i$. We see that $P(X\ge x)$ is larger for the distribution with larger value of $\lambda$, and for any two different values of $\lambda$, there is $x$ (e.g., any $x\in (0,1]$) such that $P(X\ge x)$ are not the same thus showing that the reward distributions satisfy  Assumption \ref{majorized_arms}. 

We consider an online portal that can display $K$ products because of certain limitations. Assume that the reward, which indicates the profit from the sale of a product, from each arm follows the distribution as defined in \eqref{eq:ctfn}. However, there is an additional benefit received when multiple products are sold together, e.g., reduced overhead/shipping costs. We define a  non linear reward $r$ as a function $f$ of individual arms as follows:
%
%
%
\begin{equation}
f\left(X_1, X_2,\cdots, X_K\right) = \frac{2}{K(K+1)}\sum_{i = 1}^K\sum_{j \geq i}^K X_i X_j. \label{eq:non_lin_func_arms}
\end{equation}
 The expected value of the reward in terms of expected value of rewards of individual arms is
\begin{eqnarray}
\mathbb{E}\left[f\left(X_1, X_2,\cdots, X_K\right)\right]\nonumber &=& \frac{2}{K(K+1)}\left(\sum_{i = 1}^K\mathbb{E}\left[X_i^2\right] + \sum_{i = 1}^K\sum_{j > i}^K\mathbb{E}\left[X_i\right]\mathbb{E}\left[X_j\right]\right), \label{eq:exp_non_lin_sep_ij}
\end{eqnarray}
This result follows from linearity of expectation.
We note that the expected reward is strictly increasing with respect to the expected values of individual rewards.  Figure \ref{fig:cvx_of_rewards} shows the evaluation results for setting where the reward of an action is the function described in Equation \eqref{eq:non_lin_func_arms} of the rewards of individual rewards of the arms. We see that for both values of $N$, \NAM\ outperforms UCB for $K>3$ in the time step range considered. Further, for $N=24$ and $K=2$, the gap between the proposed algorithm and UCB is small. In summary, when the value of $\binom{N}{K}$ is moderately large, and $T$ is not significantly large ($T<\Tilde{O}\left(\frac{e^{3K}N^{3K-2}}{K^{3K+3}}\right)$), \NAM\ outperforms the baseline. Further, the computation and storage complexity of the proposed algorithm are much better as compared to the baseline, as seen in Section \ref{sec:complexity}. 

In Appendix \ref{synthetic_evaluation}, we further consider two more examples for Bernoulli distribution, where the sum and maximum of rewards are considered as the reward functions. 
\section{Conclusions and Future Work} \label{conclusion}
This paper considers the problem of combinatorial multi-armed bandits with non-linear rewards, where agent chooses $K$ out of $N$ arms in each time-step and receives an aggregate reward. A novel algorithm, called \NAM, is proposed which is shown to be computationally efficient and has a space complexity which is linear in number of base arms. The algorithm is analyzed in terms of a regret bound, and is shown to outperform the approach of considering the combinatorial action as arm for limited time horizon $T$. \NAM\ provides a way to resolve two challenges in combinatorial bandits problem, the first is that the feedback is non-linear in the individual arms, and the second is that the space complexity in the previous approaches could be large due to exploding action space. \NAM\ works efficiently for large $N$ and $K$.


Followed by this work, we  provided an algorithm that  achieves a regret bound of $\Tilde{O}(K\sqrt{NKT})$ using intuitions from this paper  for the setup with non-linear Top-$K$ subset bandits with potentially correlated rewards of individual arms \citet{agarwal2020dart}. 

Considering non-symmetric functions of individual rewards and studying the applications of such  settings, such as Social Influence Maximization, provides interesting directions for future works.

\bibliography{refs.bib}

\newpage

\appendix

\section{Fundamental Lemmas}
In this subsection, we will describe some lemmas that will be later used to prove the regret bound in Theorem 1. 
The first lemma is Hoeffding's Inequality, which will be used in the results in this paper. 
\begin{lemma}[Hoeffding's Inequality \citep{krafft1969note}] \label{Hoeffding_lemma}
     Let $X$ be a  random variable bounded in $[a, b]$, and let $\bar{X}$ denote the expected value of $X$. Further, let $\hat{\bar{X}}$ denote the average of $n$ i.i.d. samples of $X$. 
     Then, for any $\epsilon>0$, we have 
    \begin{eqnarray}
        P\left(\big|\bar{X} - \hat{\bar{X}} \big| \geq \epsilon  \right) &\leq& 2e^{-\frac{2n\epsilon^2}{\left(b-a\right)^2}}
    \end{eqnarray}
\end{lemma} 

The next result shows that the property of FSD is preserved by strictly increasing functions. 
\begin{lemma} \label{ordering_on_increasing_fn}
    Suppose that a random variable $X$ has FSD over another random variable $Y$ (or 
    $X \succ Y$).  Further, let $g'(\cdot)$ be a strictly increasing function on $\mathbb{R}$. 
    Then, $g'(X)$ has FSD over $g'(Y)$, or 
    \begin{equation*}
         g'(X)\succ g'(Y).
    \end{equation*}
\end{lemma}

\begin{proof}
 Since $X \succ Y$, we have 
    \begin{eqnarray}
        P\left(X\ge x\right) &\geq& P\left(Y\ge x\right) \nonumber 
    \end{eqnarray}
    Transforming $X$ and $Y$ using the function $g'$, and using the strict monotonicity of $g'$, we have
    \begin{eqnarray}
        P\left(g'(X) \ge  g'(x)\right) &\geq& P\left(g'(Y) \ge g'(x)\right) \nonumber 
    \end{eqnarray}
    Taking $t \triangleq g'(x)$, we have,
    \begin{eqnarray}
        P\left(g'(X) \ge  t\right) &\geq& P\left(g'(Y) \ge  t\right) \nonumber 
    \end{eqnarray}
    By the same arguments, if there is an $x$ where $P\left(X\ge x\right) > P\left(Y\ge x\right)$, for $t=g'(x)$, $P\left(g'(X) \ge  t\right) > P\left(g'(Y) \ge  t\right)$. Thus, we have $g'(X) \succ g'(Y)$. 
    
\end{proof}

\section{Proof of Lemma \ref{ordering_from_expectation}} \label{correctness}

    Note we can write the actions using replacement function $h$ to replace reward of $G(j)$ arm by the reward of $G(i)$ arm, and obtain 
    \begin{eqnarray}
        {\bf d}_{{\bf a}^G_{-j}} &=& h\left({\bf d}_{{\bf a}^G_{-i}},j, X_{G(i)}\right)\text{ }\forall\text{ }j \label{eq:replace_by_i}
    \end{eqnarray}
    The expectation of the reward of an action formed using replacement can be written as expected value of the conditional expectation of the replaced reward value of arm $G(i)$. More precisely, 
    \begin{equation}
        \mathbb{E}\left[f\left({\bf d}_{{\bf a}^G_{-j}}\right)\right] = \mathbb{E}\left[\mathbb{E}\left[f\left(h\left({\bf d}_{{\bf a}^G_{-i}},j, X_{G(i)}\right)\right) \Big| X_{G(i)}\right]\right] \label{eq:reward_given_X_g_i}
    \end{equation}
    In right hand side of Equation (\ref{eq:reward_given_X_g_i}), the outer expectation is taken over $X_{G(i)}$ while the inner expectation is over ${\bf d_{a^G_{-i}}}$. In addition, we can replace the reward of an arm $G(j)$ in ${\bf d}_{{\bf a}^G_{-i}}$ by itself to obtain 
    \begin{eqnarray}
        {\bf d}_{{\bf a}^G_{-i}} &=& h\left({\bf d}_{{\bf a}^G_{-i}},j, X_{G(j)}\right) \label{eq:replace_by_j}
    \end{eqnarray}
    Similarly to \eqref{eq:reward_given_X_g_i}, we can write the expected reward of the action as the expectation of conditional expectation of reward of arm $G(j)$. Thus, we have
    \begin{equation}
        \mathbb{E}\left[f\left({\bf d}_{{\bf a}^G_{-i}}\right)\right] = \mathbb{E}\left[\mathbb{E}\left[f\left(h\left({\bf d}_{{\bf a}^G_{-i}},j, X_{G(j)}\right)\right) \Big| X_{G(j)} \right]\right] \label{eq:reward_given_X_g_j}
    \end{equation}
    The ordering between actions ${\bf a}^G_{-i}$ and ${\bf a}^G_{-j}$ is defined as an order between the expected rewards of the respective actions. We assume that ${\bf a}^G_{-i}$ provides a higher expected reward than ${\bf a}^G_{-j}$, hence we have $\mu_{{\bf a}^G_{-i}} > \mu_{{\bf a}^G_{-j}}$. This further implies 
    
    \begin{eqnarray}
        \mathbb{E}\left[f\left({\bf d}_{{\bf a}^G_{-i}}\right)\right] &>& \mathbb{E}\left[f\left({\bf d}_{{\bf a}^G_{-j}}\right)\right] \label{eq:order_arm_reward_function}
    \end{eqnarray}
    Replacing the right hand side of   \eqref{eq:order_arm_reward_function} by  \eqref{eq:reward_given_X_g_i} and the left hand side of \eqref{eq:order_arm_reward_function} by \eqref{eq:reward_given_X_g_j}, we have 
    
    \begin{eqnarray}
       && \mathbb{E}\left[\mathbb{E}\left[f\left(h\left({\bf d}_{{\bf a}^G_{-i}},j, X_{G(j)}\right)\right) \Big| X_{G(j)}\right]\right] \nonumber\\&>& \mathbb{E}\left[\mathbb{E}\left[f\left(h\left({\bf d}_{{\bf a}^G_{-i}},j, X_{G(i)}\right)\right) \Big| X_{G(i)}\right]\right] \label{eq:exp_reward_double_exp}
    \end{eqnarray}
    We define a function $g(x)$ as  the conditional expectation of reward function $f(\cdot)$ with respect to the reward of the $j^{th}$ arm being $x$. More precisely, 
    \begin{eqnarray}
         g(x) \triangleq \mathbb{E}\left[f\left(h\left({\bf d}_{{\bf a}^G_{-i}},j, X\right)\right) \Big| X = x\right] \label{eq:g_as_cond_exp}
    \end{eqnarray}
    Replacing the conditional expectations in  (\ref{eq:exp_reward_double_exp}), by $g$ as defined in  (\ref{eq:g_as_cond_exp}), we get
    \begin{eqnarray}
        \mathbb{E}_{X_{G(j)}}\left[g\left(X_{G(j)}\right)\right] &>& \mathbb{E}_{X_{G(i)}}\left[g\left(X_{G(i)}\right)\right]\label{eq:order_exp_with_g}
    \end{eqnarray}
    From Assumption 3, $g(x)$ is a strictly increasing function of $x$, and from Assumption 2, rewards of individual arms have FSD relationships. Using Assumption 3 and Assumption 2 along with equation (\ref{eq:order_exp_with_g}), we want to prove that arm $G(j)$ has FSD over arm $G(i)$. Let us assume the converse where arm $G(i)$ has FSD over arm $G(j)$, or $X_{G(i)} \succ X_{G(j)}$. 
 Due to the strict increasing nature of $g(\cdot)$ and Lemma \ref{ordering_on_increasing_fn}, we have 
 $       g\left(X_{G(i)}\right) \succ g\left(X_{G(j)}\right)$.  
Using Lemma \ref{expectation_from_majorization}, we further have
    \begin{eqnarray}
        \mathbb{E}\left[g\left(X_{G(i)}\right)\right] &\geq& \mathbb{E}\left[g\left(X_{G(j)}\right)\right] \label{eq:incorrect_exp_order}
    \end{eqnarray}
    This inequality (\ref{eq:incorrect_exp_order}) does not agree with the inequality obtained from the original assumption (\ref{eq:order_exp_with_g}) thus disproving $X_{G(i)} \succ X_{G(j)}$. Since every two arms have FSD relation and $X_{G(i)} \succ X_{G(j)}$ does not hold, we have 
    \begin{equation}
        X_{G(j)} \succ X_{G(i)}
    \end{equation}
    Hence, an ordering on the expected rewards of the $K+1$ actions constructed by leaving one arm aside, gives an ordering on the $K+1$ arms of the group.

\section{Number of time steps in SORT and MERGE subroutines} \label{sort_merge_lemma}
In this subsection, we will bound the number of exploration time steps that are spent in SORT and MERGE subroutines. The next result bounds the number of time steps spent in a group $G$ in the SORT subroutine to obtain an ordering on the actions and thus on arms (by Lemma \ref{ordering_from_expectation}). 

\begin{proof}[Proof of Lemma \ref{sort_requirements}]
	Let $G(i)$   for all $ i \in {1, 2, \cdots,K+1}$ be the different arms of group $G$ that we aim to sort using the SORT subroutine. Let $\hat{\mu}_{{\bf a}^G_{-i}}$ be the estimate of action made using all arms in $G$ except  arm $G(i)$. We want to identify the number of time steps spent in SORT subroutine, and the error probability in the ordering.   
	
	
	
	The algorithm proceeds in rounds starting from $r = 1$. We define $\Delta_r \triangleq 2^{-r}$, and each un-sorted action in round $r$ is played for $n_r \triangleq \frac{\log{2NT}}{\Delta^2_r}$ time steps. Using Hoeffding's Inequality, the expected reward estimate of each action lies in the range $\left[\hat{\mu}_{{\bf a}^G_{-i}} - \Delta_r, \hat{\mu}_{{\bf a}^G_{-i}} + \Delta_r\right]$ with probability bound given as 
	\begin{eqnarray}
	P\left(\big|\mu_{{\bf a}^G_{-i}} - \hat{\mu}_{{\bf a}^G_{-i}}\big| \geq \Delta_r  \right) &\leq& 2e^{-2\frac{\log{2NT}}{\Delta^2_r}\Delta^2_r} \nonumber \\ 
	&=& \frac{1}{2N^2T^2}\label{eq:hoeffding_error_boundl}
	\end{eqnarray}
	Thus,  the expected reward estimate of each action lies in the range $\left[\hat{\mu}_{{\bf a}^G_{-i}} - \Delta_r, \hat{\mu}_{{\bf a}^G_{-i}} + \Delta_r\right]$ with probability  at least $1- \frac{1}{2N^2T^2}$. Let it take $r_i$ rounds to be able to sort ${\bf a}^G_{-i}$ in its correct position, which implies that all other actions can be well separated from this action.  When two actions are separated, the upper confidence bound for the action with a lower estimated reward is less than the lower confidence bound of action with higher estimated reward, which gives the following inequality
	\begin{eqnarray}
	\hat{\mu}_{{\bf a}^G_{-i}} + \Delta_{r_i} &<& \hat{\mu}_{{\bf a}^G_{-j}} -\Delta_{r_i}.\nonumber 
	\end{eqnarray}
	This further means that the actions ${\bf a}^G_{-i}$ and ${\bf a}^G_{-j}$ were inseparable in round $r_i -1$, thus the following holds.
	\begin{eqnarray}
	\hat{\mu}_{{\bf a}^G_{-i}} + \Delta_{r_i-1} &\geq& \hat{\mu}_{{\bf a}^G_{-j}} -\Delta_{r_i-1} \label{eq:estimate_deltas_r_1}
	\end{eqnarray}
	However, using Lemma \ref{Hoeffding_lemma}, $\hat{\mu}_{{\bf a}^G_{-i}}$ lies between the confidence region around the true mean as $\hat{\mu}_{{\bf a}^G_{-i}} \in \left[\mu_{{\bf a}^G_{-i}} - \Delta_{r_i-1}, \mu_{{\bf a}^G_{-i}} + \Delta_{r_i-1}\right]$. The same is true for $\hat{\mu}_{{\bf a}^G_{-j}}$. Using the upper limits of estimated mean rewards, we can rewrite inequality (\ref{eq:estimate_deltas_r_1}) as 
	\begin{eqnarray}
	\mu_{{\bf a}^G_{-i}} + 2\Delta_{r_i-1} &\geq&\mu_{{\bf a}^G_{-j}} - 2\Delta_{r_i-1} \label{eq:true_deltas}
	\end{eqnarray} 
	From inequality (\ref{eq:true_deltas}), we get $\Delta_{r_i}$ in terms of the difference of expected rewards of the two arms as follows. 
	\begin{eqnarray}
	4\Delta_{r_i-1} &\geq& \mu_{{\bf a}^G_{-j}}-\mu_{{\bf a}^G_{-i}} \nonumber \\ 
	&\geq& \frac{1}{U}\left(\mathbb{E}\left[X_{G(i)}\right] - \mathbb{E}\left[X_{G(j)}\right]\right), \label{eq:cont_arm_diff}
	\end{eqnarray} 
	where the last step follows from Corollary \ref{double_sided_continuity}. 
	Using the definition of $\Delta_r$ and the upper bound obtained in inequality (\ref{eq:cont_arm_diff}), we have
	\begin{eqnarray}
	2^{-(r_i-1)} &=& \Delta_{r_i-1}\nonumber \\ 
	&\geq& \frac{\left(\mathbb{E}\left[X_{G(i)}\right] - \mathbb{E}\left[X_{G(j)}\right]\right)}{4U} \nonumber 
	\end{eqnarray} 
	Taking logarithm with base $2$, we obtain a lower bound on the number of rounds needed as 
	\begin{eqnarray}
	r_i - 1 &\leq& \log_2{\left(\frac{4U}{\left(\mathbb{E}\left[X_{G(i)}\right] - \mathbb{E}\left[X_{G(j)}\right]\right)}\right)}. \nonumber 
	\end{eqnarray} 
	
	The term $\left(\mathbb{E}\left[X_{G(i)}\right] - \mathbb{E}\left[X_{G(j)}\right]\right)$ will be lowest for $j=i+1$ or $i-1$ since the arms closest to $i$ will result in the lowest gap. Thus, we define  $\delta_i$ for arm $G(i)$ as 
	\begin{eqnarray}
	\delta_i &\triangleq& \min\left\{\mathbb{E}\left[X_{G(i)}\right] - \mathbb{E}\left[X_{G(i+1)}\right],\right.\\&&\left. \mathbb{E}\left[X_{G(i-1)}\right] - \mathbb{E}\left[X_{G(i)}\right]\right\} \nonumber 
	\end{eqnarray}
	The number of rounds to correctly rank arm $G(i)$ can thus be upper bounded as follows. 
	\begin{eqnarray}
	r_i &\leq& \log_2{\left(\frac{8U}{\delta_i}\right)}. \nonumber 
	\end{eqnarray}
	Having bounded the number of rounds to rank ${\bf a}^G_{-i}$, we  find the number of time steps each action is played till round $r_i$ as follows.
	%
	\begin{eqnarray*}
	n_{r_i} &=& \frac{\log{2NT}}{\Delta^2_{r_i}} \label{eq:n_r_i_Delta}\\
	&=& 2^{2{r_i}}\log{2NT} \label{eq:n_r_i_rounds}\\
	&\le& 2^{2\log_2{\left(\frac{8U}{\delta_i}\right)}}\log{2NT} \label{eq:n_r_i_exp_delta}\\
	&=& \frac{64U^2\log{2NT}}{\delta^2_i} \label{eq:n_r_i_frac_delta}
	\end{eqnarray*}
	This provides the number of times  action ${\bf a}^G_{-i}$ is played. Let $\lambda$ be chosen as a threshold, which is the precision level below which we cannot correctly order two actions. Using  $\lambda$ as a lower bound for $\delta_i$, each of the $K+1$ action is chosen for $n_{r_i}$ times, where $\delta_i$ is replaced by $\max(\delta_i,\lambda)$. Thus, the total time steps any of the action is selected is given as 
	
	\begin{eqnarray*}
	n' &\le& \sum_{i=1}^{K+1}\frac{64U^2\log{2NT}}{\max(\delta^2_i, \lambda^2)}   \end{eqnarray*}
	which proves the number of time steps as in the statement of the Lemma.
	
	
	For $K+1$ arms, there can be $\binom{K+1}{2}$ ordered pairs denoting the edges of a complete graph, where each edge is the ordering of the two vertices. However, the ordering is preserved by a sub-graph which is a chain by just keeping the edges which connect two immediately ordered vertices. The number of edges thus becomes $(K+1) - 1 = K$. Thus, there are $K$ orderings in a group, and all of them should be correct to make the algorithm work. The probability of error in sorting  group $G$ can then be written as
	\begin{eqnarray*}
	&&P\left(\{\text{any arm is incorrectly sorted}\}\right)\\ &=& 1 - P\left(\{\text{all arms are correctly sorted}\}\right)\\
	&\le& 1 - \left(1-\frac{1}{2N^2T^2}\right)^K\\
	&<& 1 - \left(1 -\frac{K}{2N^2T^2} \right)\\
	&=& \frac{K}{2N^2T^2}, \label{eq:sort_prob_error}
	\end{eqnarray*} 
	where the probability that any two arms are incorrectly sorted is bounded by $\frac{1}{2N^2T^2}$ as given in \eqref{eq:hoeffding_error_boundl}. This proves the probability of correct ordering as in the statement of the Lemma. 
\end{proof}

Having understood the number of time steps spent in the SORT subroutine, and the error probability of ranking arms in each group, we now consider the MERGE subroutine. In the following lemma, we find the number of time steps it takes to merge two groups, with an error probability on the ordering in the combined group.  
\begin{proof}[Proof of Lemma \ref{Merge_requirements}]
In the MERGE subroutine, we use the sorted groups $G_1$ and $G_2$ to construct a group $G$ of $K$ elements such that arms in group $G$ are sorted. We maintain two counters $i$ and $j$ for the groups $G_1$ and $G_2$ respectively. We replace arm $G_2(j)$ by $G_1(i)$ in group $G_1$ to create a new action. We play both actions in rounds starting from $r=1$. At the end of  round $r$, the deviation of estimated mean reward and true expected rewards is $\Delta_r \triangleq 2^{-r}$. By round $r$, each action has been played for $n_r \triangleq \frac{\log{2NT}}{\Delta_r^2}$ time steps similar to Algorithm 2. Two actions are separated when the upper confidence bound of worse action is lesser than the lower confidence bound of better action. We then add the arm corresponding to better action to group $G$ and increment the counter for the corresponding group from which arm was picked and continue comparing new actions. This is continued till we have $K$ arms in $G$.
	
Since the fundamental concept of comparison of two actions in  Algorithm 2 is the same as that in  Algorithm 3, similar analysis follows and the number of time steps required to merge the groups is 
	\begin{eqnarray*}
	n' &\le& \sum_{i=1}^{K+1} \frac{64 U^2 \log 2NT}{\max{(\delta_{G(i)}^2,\lambda^2)}}, 
	\end{eqnarray*}
	where $G(K+1)$ is the last arm with which comparison was made but not added to group $G$. The same argument as in the proof of Lemma \ref{sort_requirements} can be used to bound the error probability of a single run of Algorithm 3 by $\frac{K}{2N^2T^2}$.
\end{proof}

\section{Proof of Lemma \ref{algorithm_error_event_prob}}\label{proof_algorithm_error_event_prob}
To bound the probability of error of Algorithm 1, we define the event $\mathcal{E}$ which is the event when the algorithm makes an error. The algorithm makes an error when either the SORT subroutine, the MERGE subroutine, or both make an error. So, we can write the error event as a union of error events in sorting and error events in merging. Let $\mathcal{E}_{s}$ the event where at least one of the $\frac{N}{K+1}$ calls made to Algorithm 2 resulted in an incorrect list, and $\mathcal{E}_m$ be the event where at least one of the $\frac{N}{K+1} - 1$ merges is in error.
\begin{eqnarray*}
	\mathcal{E} &=& \mathcal{E}_s \cup \mathcal{E}_m
\end{eqnarray*}
Since the probability of the union of events is upper bounded by the sum of probabilities of individual events, we get
\begin{eqnarray}
P\left(\mathcal{E}\right) &\leq& P\left(\mathcal{E}_s\right) + P\left(\mathcal{E}_m\right) \label{ref:algo_error_sum}
\end{eqnarray}
We now identify upper bounds for $P\left(\mathcal{E}_s\right)$ and $P\left(\mathcal{E}_m\right)$ by breaking down the error events into error in each call to SORT and MERGE subroutine.

Let us define an event $\mathcal{E}_{s,i}$ which denotes that there was an error in the sorted list given by Algorithm 2 for $i^{th}$ group, or 
\begin{eqnarray}
\mathcal{E}_{s,i} &\triangleq& \left\{\text{at least two arms are incorrectly placed in }\right.\nonumber\\&&\left.\text{sorting } i^{th} \text{ group}\right\}
\end{eqnarray}

Hence, $\mathcal{E}_{s}$ is a union of $\mathcal{E}_{s,i}$ for all  $i \in \{1, 2, \cdots, \frac{N}{K+1}\}$, or  $\mathcal{E}_s = \bigcup_{i = 1}^{\frac{N}{K+1}}\mathcal{E}_{s,i}$. 
Probability of error in any of the sorting operations is thus given as
\begin{eqnarray}
P\left(\mathcal{E}_s\right) &=& P\left(\bigcup_{i = 1}^{\frac{N}{K+1}}\mathcal{E}_{s,i}\right) \label{eq:sort_union_bound}\\
&\leq& \sum^{\frac{N}{K+1}}_{i = 1}P\left(\mathcal{E}_{s,i}\right)\label{eq:sort_union_bound_sum}\\
&<& \sum^{\frac{N}{K+1}}_{i = 1}\frac{K}{2NT^2} \label{eq:replace_prob_sort_val}\\
&=& \frac{N}{K+1}\frac{K}{2N^2T^2}\\
&<& \frac{1}{2NT^2}
\end{eqnarray}
where (\ref{eq:sort_union_bound}), and (\ref{eq:sort_union_bound_sum}) follow from the definition of events and the union bound, respectively, and  (\ref{eq:replace_prob_sort_val}) follows from Lemma \ref{sort_requirements}.

Similarly we define event $\mathcal{E}_{m,i}$ representing that Algorithm 2 incorrectly merges the merged group up to group $i$ and $(i+1)^{th}$ group. Let $G_M^i$ be the merged sorted group formed by merging sorted groups $G_M^{i-1}$ and $G_i$ for $i>1$, with $G_M^1 = G_1$. Then, we have
\begin{eqnarray*}
	\mathcal{E}_{m,i} = \left\{\text{error occurred while merging } G_M^{i} \text{ and } G_{i+1}\right\} . 
\end{eqnarray*}

We note that $\mathcal{E}_{m}$ is a union of $\mathcal{E}_{m,i}$ for all $i \in \{1, 2, \cdots, \frac{N}{K+1}-1\}$, or   $ \mathcal{E}_m = \bigcup_{i = 1}^{\frac{N}{K+1}-1}\mathcal{E}_{m,i}$. 
Probability of error in the MERGE subroutine is given as 
\begin{eqnarray}
P\left(\mathcal{E}_m\right) &=& P\left(\bigcup_{i = 1}^{\frac{N}{K+1}-1}\mathcal{E}_{m,i}\right) \label{eq:merge_union_bound}\\
&\leq& \sum^{\frac{N}{K+1}-1}_{i = 1}P\left(\mathcal{E}_{m,i}\right)\label{eq:merge_union_bound_sum}\\
&<& \sum^{\frac{N}{K+1}-1}_{i = 1}\frac{K}{2N^2T^2} \label{eq:replace_prob_merge_val}\\
&<& \frac{N}{K+1}\frac{K}{2N^2T^2}\\
&<& \frac{1}{2NT^2}
\end{eqnarray}
where (\ref{eq:merge_union_bound}) and (\ref{eq:merge_union_bound_sum}) follow from the definition of events and union bound, respectively, and (\ref{eq:replace_prob_merge_val}) follows from Lemma \ref{Merge_requirements}.

Substituting bounds obtained on $P\left(\mathcal{E}_m\right)$ and $P\left(\mathcal{E}_s\right)$ in  (\ref{ref:algo_error_sum}), we have
\begin{eqnarray}
P\left(\mathcal{E}\right) &<& \frac{1}{2NT^2} + \frac{1}{2NT^2}\\
&=& \frac{1}{NT^2}
\end{eqnarray}
The total error probability of the algorithm is thus bounded by $\frac{1}{NT^2}$, proving the statement of the Lemma. 

\section{Proof of Lemma \ref{exploration_time_bound}}\label{proof_exploration_time_bound}
  Exploration in Algorithm 1 is done using SORT and MERGE subroutines, so we will analyze the maximum time taken by the two subroutines. To sort all groups, Algorithm 1 runs SORT $\frac{N}{K+1}$ times, and to merge the groups, Algorithm 1 runs  $\frac{N}{K+1}-1$ times. Let $i^{th}$ run of SORT uses $T_{SORT, i}$ time steps, and $j^{th}$ run of MERGE uses $T_{MERGE,j}$ time steps. The total number of time-steps used for exploration can be written as
\begin{eqnarray*}
	T_{exp} &=& \sum_{s = 1}^{\frac{N}{K+1}}T_{SORT,s} + \sum_{m = 1}^{\frac{N}{K+1}-1}T_{MERGE,m}\\
	&\leq& \sum_{s = 1}^{\frac{N}{K+1}}\max_{s}{\left(T_{SORT,s}\right)} + \sum_{m = 1}^{\frac{N}{K+1}-1}\max_{m}{\left(T_{MERGE,m}\right)}\\
	&=& \frac{N}{K+1}\max_{s}{\left(T_{SORT,s}\right)} \nonumber\\&&+ \left(\frac{N}{K+1}-1\right)\max_{m}{\left(T_{MERGE,m}\right)}
\end{eqnarray*}
We use Lemma \ref{sort_requirements} to find $\max_{s}{\left(T_{SORT,s}\right)}$ as follows. 
\begin{eqnarray*}
	&&T_{SORT,s} \nonumber\\ &=& \sum_{j=1}^{K+1}\frac{64U^2\log{2NT}}{\left(\max{\left(\lambda,\delta_{s,j}\right)}\right)^2}\\
	&\leq& \sum_{j, \delta_{s,j} > \lambda}^{K+1}\frac{64U^2\log{2NT}}{\lambda^2} + \sum_{j, \delta_{s,j} < \lambda}^{K+1}\frac{64U^2\log{2NT}}{\lambda^2} \label{eq:lower_bound_delta}\\
	&=& \sum_{j=1}^{K+1}\frac{64U^2\log{2NT}}{\lambda^2}\\
	&=& \frac{64\left(K+1\right)U^2\log{2NT}}{\lambda^2}.  \label{eq:sort_time_bound}
\end{eqnarray*}
Similarly,  we use Lemma \ref{Merge_requirements}  to find $\max_s\left(T_{MERGE_s}\right)$ as follows. 
\begin{eqnarray*}
	&&T_{MERGE,m} \nonumber\\
	&=& \sum_{j=1}^{K+1}\frac{64U^2\log{2NT}}{\left(\max{\left(\lambda,\delta_{m,j}\right)}\right)^2} \\
	&\leq& \sum_{j, \delta_{m,j} > \lambda}^{K+1}\frac{64U^2\log{2NT}}{\lambda^2} + \sum_{j, \delta_{m,j} < \lambda}^{K+1}\frac{64U^2\log{2NT}}{\lambda^2} \label{eq:lower_bound_delta_merge}\\
	&=& \sum_{j=1}^{K+1}\frac{64U^2\log{2NT}}{\lambda^2}\\
	&=& \frac{64\left(K+1\right)U^2\log{2NT}}{\lambda^2}.
\end{eqnarray*}
Total exploration time can now be bounded by using the values for maximum time taken for SORT and MERGE as,
\begin{eqnarray*}
	&&T_{exp} \nonumber\\
	&\leq& \frac{N}{K+1}\max_{s}{\left(T_{SORT,s}\right)} + \left(\frac{N}{K+1}-1\right)\max_{s}{\left(T_{SORT,s}\right)}\\
	&\leq& \frac{N}{K+1}\frac{64\left(K+1\right)U^2\log{2NT}}{\lambda^2} \nonumber\\&&+ \left(\frac{N}{K+1}-1\right)\frac{64\left(K+1\right)U^2\log{2NT}}{\lambda^2}\\
	&<& \frac{N}{K+1}\frac{64\left(K+1\right)U^2\log{2NT}}{\lambda^2} \nonumber\\&&+ \left(\frac{N}{K+1}\right)\frac{64\left(K+1\right)U^2\log{2NT}}{\lambda^2}\\
	&\leq& \frac{64NU^2\log{2NT}}{\lambda^2} + \frac{64NU^2\log{2NT}}{\lambda^2}\\
	&\leq& \frac{128NU^2\log{2NT}}{\lambda^2}.
\end{eqnarray*}

\section{Proof of Lemma \ref{exploitation_regret_bound}}\label{proof_exploitation_regret_bound}
We will bound the expected regret where the suboptimal action is returned by \NAM, which includes an error from Algorithm \ref{alg:da} or from Algorithm \ref{alg:ma}. Let the chosen action be ${{\bf \hat{a}}^*} = (\hat{a}^*_1, \cdots, \hat{a}^*_K)$ and  optimal action be ${\bf a}^* = ({a}^*_1, \cdots, {a}^*_K)$. Then, the gap in the actions is $P=\sqrt{\min_{\Pi}\{\sum_{i=1}^K ({\bf d}_{{ \hat{a}}^*_i} - {\bf d}_{{a}^*_{\Pi(i)}})^2\}}$, where $\Pi$ is any possible permutation of $\{1, \cdots, K\}$. Also, let $\Pi^*$ be the permutation that optimizes the above minimization. Using Assumption 4, and $U$ defined in Corollary \ref{double_sided_continuity}, we have 
\begin{eqnarray}
\mathbb{E}\left[r_{{\bf a}^*} - r_{{\bf \hat{a}}^*}\right] &\leq& U P  \label{eq:pbound}
\end{eqnarray}

We consider two possible cases. Case 1 corresponds to the scenario where for some $i$, {$\mathbb{E}\left[{\bf d}_{{\hat{a}}^*_i}\right] - \mathbb{E}\left[{\bf d}_{{a}^*_{\Pi^*(i)}}\right] >\lambda$}. The second case is when for all $i$, {$\mathbb{E}\left[{\bf d}_{{\hat{a}}^*_i}\right] - \mathbb{E}\left[{\bf d}_{{a}^*_{\Pi^*(i)}}\right] \le \lambda$}.


{\bf Case 1: For some $i$, {$\mathbb{E}\left[{\bf d}_{{\hat{a}}^*_i}\right] - \mathbb{E}\left[{\bf d}_{{a}^*_{\Pi^*(i)}}\right] >\lambda$}. } In this case, the incorrect action has arms which could have been separated without hitting the threshold $\lambda$, but did not do so because of an error in the SORT or MERGE subroutine. We note that since $\mathbb{E}\left[{\bf d}_{{\hat{a}}^*_i}\right]$ and $\mathbb{E}\left[{\bf d}_{{a}^*_{\Pi^*(i)}}\right]$ are both in $[0,1]$, $P\le \sqrt{K}$. The expected regret at time $t$  can be written as
%
%
%
\begin{eqnarray}
\mathbb{E}\left[R(t)\right] &=& \mathbb{E}\left[r_{{\bf a}^*} - r_{{\bf \hat{a}}^*}\big|{\bf a}_t \neq {\bf a}^*\right]  \times P\left(\left\{{\bf a}_t \neq {\bf a}^*\right\}\right) \label{eq:conditional_regret}\\
&\leq& U\sqrt{K} \times P\left(\left\{{\bf a}_t \neq {\bf a}^*\right\}\right) \label{eq:cond_regret_expd}\\
&\le& \frac{U\sqrt{K}}{NT^2} \label{eq:cond_regret_final}, 
\end{eqnarray}
where (\ref{eq:cond_regret_expd}) follows from \eqref{eq:pbound} with $P\le \sqrt{K}$, and \eqref{eq:cond_regret_final} follows from \cref{algorithm_error_event_prob}.

{\bf Case 2:  For all $i$, {$\mathbb{E}\left[{\bf d}_{{\hat{a}}^*_i}\right] - \mathbb{E}\left[{\bf d}_{{a}^*_{\Pi^*(i)}}\right] \le \lambda$}.} In this case,  the SORT or MERGE subroutines will not be able to differentiate between the two actions. Thus, $P\left(\left\{{\bf a}_t \neq {\bf a}^*\right\}\right) $ will no longer be bounded by $1/{NT^2}$ in this case, since the individual rewards are bounded by $\lambda$, we have $P\le \lambda \sqrt{K}$.  Thus, we have 


\begin{eqnarray}
\mathbb{E}\left[R(t)\right] &=& \mathbb{E}\left[r_{{\bf a}^*} - r_{{\bf \hat{a}}^*}\big|{\bf a}_t \neq {\bf a}^*\right] \times P\left(\left\{{\bf a}_t \neq {\bf a}^*\right\}\right) \label{eq:conditional_regret_good_event}\\
&\leq& U\lambda \sqrt{K} \times P\left(\left\{{\bf a}_t \neq {\bf a}^*\right\}\right) \label{eq:cond_regret_expd_good_event}\\
&\le & U\lambda \sqrt{K}, \label{eq:cond_regret_final_good_event}
\end{eqnarray}
where  (\ref{eq:cond_regret_expd_good_event}) follows from \eqref{eq:pbound} and $P\le \lambda \sqrt{K}$.


Combining both the cases, the maximum regret the algorithm can incur is bounded by $U\lambda\sqrt{K} + \frac{U\sqrt{K}}{NT^2}$, thus proving the result.

\section{Complete Implementation of SORT and MERGE subroutines} \label{detail_sm_algos}
The detailed  SORT and MERGE subroutines can be seen in Algorithm \ref{alg:detailed_da}-\ref{alg:update_mean}. 
\begin{algorithm}[!htb]
	\small
	\begin{algorithmic}[1]
        \Procedure{Sort}{$G, \lambda$, T, N, K} \Comment{Group of $K+1$ arms, separation threshold}
		\State Initialize $g_i := {\bf a}^G_{-i}\ \ \forall\ 1\leq i\leq K+1$ \Comment{Rename to reduce notation clutter}
		\State Initialize $unsorted \leftarrow G$  \Comment{Set of unsorted arms}
		\State Initialize $G^*[1:K+1]\leftarrow 0$ \Comment{Ranked list of arms initialized with zeros}
		\State Initialize $t[1:K+1]\leftarrow 0$ \Comment{Array to store number of times each action is played}
        \State Initialize $\hat{\mu}_{g_i}\leftarrow 0$; $r\leftarrow 1$; $\Delta_r\leftarrow\frac{1}{2}$ $n_r =\frac{2\log{(TNK)}}{\Delta^2}$
	    \While{$\left(\Delta_r > \lambda\right)$ \textbf{and} $unsorted \neq \phi$}
            \For{$i \in unsorted$}
	                \State $\hat{\mu}_{g_i}, t[i] =$ \Call{Update\_Mean}{$\hat{\mu}_{g_i}, {g_i}, t[i], n_r$} \Comment{{Exploration happens here}}
            \EndFor
            \State $j = \{\arg $ sort$ (\hat{\mu}_{g_i})\}$
            \For{$1 \leq k \leq K+1$}
                \If{$\hat{\mu}_{g_{j[k]}} < \hat{\mu}_{g_{j[k-1]}} - 2\Delta_r$ \textbf{and} $\hat{\mu}_{g_{j[k]}} > \hat{\mu}_{g_{j[k+1]}} + 2\Delta_r$}
                    \State $G^*[K+1-(k-1)] = G[j[k]]$;  $unsorted = unsorted\setminus\{j[k]\}$
                \EndIf\Comment{Perform only relevant checks for $k\in\{1, K+1\}$}
            \EndFor
            \State $r=r+1$; $\Delta_r=\frac{\Delta_{r-1}}{2}$ ; $n_r =\frac{2\log{(TNK)}}{\Delta^2}$\Comment{Update round parameters}
        \EndWhile
        \For{$k\in unsorted$}
            \State $G^*[K+1-k] = G[j[k]]$;  $unsorted = unsorted\setminus\{j[k]\}$
        \EndFor
        \State \textbf{return} $G^*[1:K]$ \Comment{Optimal K arms increasing sorted increasingly in expected individual rewards}
        \EndProcedure
	\end{algorithmic}
	\caption{SORT}\label{alg:detailed_da}
\end{algorithm}
\begin{algorithm}[!htb]
	\small
	\begin{algorithmic}[1]
	    \Procedure{Merge}{$G_1, G_2, \lambda$, T, N, K}\Comment{Groups of $K$ arms each, and precision}
		\State Initialize $G^*[1:K]\leftarrow 0$ \Comment{Array of size K initialized with $0$ to store the optimal arms}
        \State Initialize $\hat{\mu}_{G_1} \leftarrow 0;\ \hat{\mu}_{A} \leftarrow 0$;  $i\leftarrow 1;\ j\leftarrow 1$
        \State Initialize $t_1\leftarrow0$; $r_1\leftarrow 1$; $\Delta_{r_1}\leftarrow\frac{1}{2}$ ; $n_{r_1} =\frac{2\log{(TNK)}}{\Delta_{r_1}^2}$
	    \For{$k=1:K$}
            \State Construct new action by replacing $i^{th}$ arm of $G_1$ by $j^{th}$ arm of $G_2$
            \begin{equation}
                    {\bf a}_{i,j} = \left(G_1 \setminus \{G_1(i)\}\right) \cup \{G_2(j)\}
                  \end{equation}
            \State Initialize $t_2\leftarrow0$; $r_2\leftarrow 1$; $\Delta_{r_2}\leftarrow\frac{1}{2}$ ; $n_{r_2} =\frac{2\log{(TNK)}}{\Delta_{r_2}^2}$ \Comment{For every new action constructed}
		    \While{$\left(\Delta_{r_2} > \lambda\right)$ \textbf{and} $\left(G^*[k] == 0\right)$}
                \State $\hat{\mu}_{{\bf a}^{G_1}}, t_1 =$ \Call{Update\_mean}{$\hat{\mu}_{{\bf a}^{G_1}}, {{\bf a}^{G_1}}, t_1, n_{r_1}$} \Comment{{Exploration happens here}}
                \State $\hat{\mu}_{{\bf a}_{i,j}}, t_2 =$ \Call{Update\_mean}{$\hat{\mu}_{{\bf a}_{i,j}}, {{{\bf a}_{i,j}}}, t_2, n_{r_2}$} \Comment{{Exploration happens here}}

                \If{$\hat{\mu}_{{\bf a}^{G_1}} < \hat{\mu}_{{\bf a}_{i,j}} - 2\Delta_{r_1}$}
                    \State $G^*[k] = G_2(j)$;  $j = j+1$
                \ElsIf{$\hat{\mu}_{{\bf a}^{G_1}} > \hat{\mu}_{{\bf a}_{i,j}} + 2\Delta_{r_1}$}
                        \State $G^*[k] = G_1(i)$;  $i = i+1$
                \EndIf
                \State $r_2=r_2+1$; $\Delta_{r_2}=\frac{\Delta_{r_2-1}}{2}$ ; $n_{r_2} =\frac{2\log{(TNK)}}{\Delta_{r_2}^2}$
                \If{$r_2 > r_1$}
                    \State $r_1=r_1+1$; $\Delta_{r_1}=\frac{\Delta_{r_1-1}}{2}$ ; $n_{r_1} =\frac{2\log{(TNK)}}{\Delta_{r_1}^2}$
                \EndIf
            \EndWhile
            \If{$G*[k] == 0$}\Comment{If the arm have less separation than $\lambda$}
                \If{$\hat{\mu}_{{\bf a}^{G_1}} < \hat{\mu}_{{\bf a}_{i,j}}$}
                    \State $G^*[k] = G_2(j)$;  $j = j+1$
                \ElsIf{$\hat{\mu}_{{\bf a}^{G_1}} > \hat{\mu}_{{\bf a}_{i,j}}$}
                        \State $G^*[k] = G_1(i)$;  $i = i+1$
                \EndIf
            \EndIf
            \State $k = k+1$
		\EndFor
        \State \textbf{return} $G^*$\Comment{Merged set of $G_1$, and $G_2$}
        \EndProcedure
	\end{algorithmic}
	\caption{MERGE }\label{alg:detailed_ma}
\end{algorithm}
\begin{algorithm} [!thb]
\caption{UPDATE\_MEAN}\label{alg:update_mean}
	\small
	\begin{algorithmic}[1]
	    \Procedure{Update\_mean}{$\mu$, ${\bf a}$, $t$, $n$}\Comment{{Explore action ${\bf a}$ by playing it $n-t$ times}}
        \While{$t \leq n$}
            \State $r_{{\bf a},t}$ = Reward collected by playing ${\bf a}$
            \State $\hat{\mu}_{\bf a} = \frac{t\times \hat{\mu}_{\bf a} + r_{{\bf a},t}}{t+1}$
            \State $t= t+1$
        \EndWhile
        \State {\bf return} $\hat{\mu}_{\bf a}, t$
        \EndProcedure
	\end{algorithmic}
\end{algorithm}



\section{Synthetic Evaluation Results for Bernoulli Rewards} \label{synthetic_evaluation}
In this section, we evaluate \NAM\ under multiple synthetic problem settings. We compare the result with improved UCB algorithm as described in \citep{auer2010ucb}. Since this paper provides the first result with non-linear reward functions for CMAB problem with bandit feedback, we compare with the UCB algorithm \citep{auer2010ucb} which is optimal for small $N$ and $K$ while having the regret scale with $\binom{N}{K}$. 


For evaluations, we ran the algorithm for $T = 10^6$ time steps and averaged over $30$ runs. We compare cumulative regret at each $t$ starting from $t=0$, which is defined as,
\begin{eqnarray*}
    W(t) = \sum^t_{t'=0}R(t')
\end{eqnarray*}

We consider two values of $N \in \{12, 24\}$. For $N=12$, we choose $K \in \{2, 3, 5\}$, while for $N=24$, we choose $K \in \{2, 3, 5, 7, 11\}$. Since the arms must have FSD over each other, we describe a  example single parameter distributions for the reward of each arm that have this property.  Reward of arm $i$ comes from the set $\{0,1\}$ and follows a Bernoulli distribution with parameter $p_i$, or 
        \begin{eqnarray*}
            X_i &\sim& Bern(p_i).
        \end{eqnarray*}
        We note that arm $i$ has FSD over arm $j$ if $p_i>p_j$. Thus, this reward distribution satisfies Assumption \ref{majorized_arms} as long as no two arms have same parameter, or $p_i\ne p_j$ for any $i\ne j$. Figure \ref{fig:bernoulli} plots $P(X\ge x)$ of the reward function for different values of $p_i$. We see that $P(X\ge x)$ is larger for the distribution with larger value of $p$, and for any two different values of $p$, there is $x$ (e.g., any $x\in (0,1]$) such that $P(X\ge x)$ are not the same thus showing that the reward distributions satisfy  Assumption \ref{majorized_arms}.


\begin{figure*}
    \centering
    \subfigure[Bernoulli Rewards]{
\begin{tikzpicture}[scale=0.75, every node/.style={scale=0.85}]

\definecolor{color1}{rgb}{0,0,0}
\definecolor{color2}{rgb}{0.75,0.75,0}
\definecolor{color3}{rgb}{0.75,0,0.75}

\begin{axis}[
grid=both,
grid style={white!50!black},
axis background/.style={fill=white},
axis line style={black},
x label style={at={(axis description cs:0.5,-0.025)},anchor=north},
y label style={at={(axis description cs:-0.1,.5)},anchor=south},
legend cell align={left},
legend entries={{$p_i$ = 0.1},{$p_i$ = 0.3},{$p_i$ = 0.5},{$p_i$ = 0.7},{$p_i$ = 0.9}},
legend style={at={(0.06,0.03)}, anchor=south west, draw=white!80.0!black, fill=white!99.80392156862746!black},
tick align=outside,
tick pos=left,
xlabel={x},
xmajorgrids,
xmin=-0.039, xmax=1.039,
ylabel={$P(X \geq x)$},
ymajorgrids,
ymin=-0.0598185, ymax=1.0400715
]
\addlegendimage{thick, mark=asterisk, green!50.0!black}
\addlegendimage{thick, mark=o, color1}
\addlegendimage{thick, mark=star, color2}
\addlegendimage{thick, mark=triangle, lightgray!62.22222222222222!black}
\addlegendimage{thick, mark=diamond, color3}
\addplot [thick, black]
table [row sep=\\]{%
-0.3000000000000000	1.0 \\
0.00000000000000000	1.0 \\
};
\addplot [thick, mark=*, black]
table [row sep=\\]{%
0.0000000000000000	1.0 \\
};
\addplot [thick, mark=o, green!50.0!black, solid]
table [row sep=\\]{%
0.00	0.09983 \\
};
\addplot [thick, mark=*, green!50.0!black, solid]
table [row sep=\\]{%
1.00	0.09983 \\
};
\addplot [thick, green!50.0!black, mark=asterisk, mark repeat=10]
table [row sep=\\]{%
0.01	0.09983 \\
0.03	0.09983 \\
0.05	0.09983 \\
0.07	0.09983 \\
0.09	0.09983 \\
0.11	0.09983 \\
0.13	0.09983 \\
0.15	0.09983 \\
0.17	0.09983 \\
0.19	0.09983 \\
0.21	0.09983 \\
0.23	0.09983 \\
0.25	0.09983 \\
0.27	0.09983 \\
0.29	0.09983 \\
0.31	0.09983 \\
0.33	0.09983 \\
0.35	0.09983 \\
0.37	0.09983 \\
0.39	0.09983 \\
0.41	0.09983 \\
0.43	0.09983 \\
0.45	0.09983 \\
0.47	0.09983 \\
0.49	0.09983 \\
0.51	0.09983 \\
0.53	0.09983 \\
0.55	0.09983 \\
0.57	0.09983 \\
0.59	0.09983 \\
0.61	0.09983 \\
0.63	0.09983 \\
0.65	0.09983 \\
0.67	0.09983 \\
0.69	0.09983 \\
0.71	0.09983 \\
0.73	0.09983 \\
0.75	0.09983 \\
0.77	0.09983 \\
0.79	0.09983 \\
0.81	0.09983 \\
0.83	0.09983 \\
0.85	0.09983 \\
0.87	0.09983 \\
0.89	0.09983 \\
0.91	0.09983 \\
0.93	0.09983 \\
0.95	0.09983 \\
0.97	0.09983 \\
0.99	0.09983 \\
};
\addplot [thick, mark=o, color1, solid]
table [row sep=\\]{%
0.00	0.29842 \\
};
\addplot [thick, mark=*, color1, solid]
table [row sep=\\]{%
1.00	0.29842 \\
};
\addplot [thick, color1, mark=o, mark repeat=10]
table [row sep=\\]{%
0.01	0.29842 \\
0.03	0.29842 \\
0.05	0.29842 \\
0.07	0.29842 \\
0.09	0.29842 \\
0.11	0.29842 \\
0.13	0.29842 \\
0.15	0.29842 \\
0.17	0.29842 \\
0.19	0.29842 \\
0.21	0.29842 \\
0.23	0.29842 \\
0.25	0.29842 \\
0.27	0.29842 \\
0.29	0.29842 \\
0.31	0.29842 \\
0.33	0.29842 \\
0.35	0.29842 \\
0.37	0.29842 \\
0.39	0.29842 \\
0.41	0.29842 \\
0.43	0.29842 \\
0.45	0.29842 \\
0.47	0.29842 \\
0.49	0.29842 \\
0.51	0.29842 \\
0.53	0.29842 \\
0.55	0.29842 \\
0.57	0.29842 \\
0.59	0.29842 \\
0.61	0.29842 \\
0.63	0.29842 \\
0.65	0.29842 \\
0.67	0.29842 \\
0.69	0.29842 \\
0.71	0.29842 \\
0.73	0.29842 \\
0.75	0.29842 \\
0.77	0.29842 \\
0.79	0.29842 \\
0.81	0.29842 \\
0.83	0.29842 \\
0.85	0.29842 \\
0.87	0.29842 \\
0.89	0.29842 \\
0.91	0.29842 \\
0.93	0.29842 \\
0.95	0.29842 \\
0.97	0.29842 \\
0.99	0.29842 \\
};
\addplot [thick, mark=o, color2, solid]
table [row sep=\\]{%
0.00	0.5006 \\
};
\addplot [thick, mark=*, color2, solid]
table [row sep=\\]{%
1.00	0.5006 \\
};
\addplot [thick, color2, mark=star, mark repeat=10]
table [row sep=\\]{%
0.01	0.5006 \\
0.03	0.5006 \\
0.05	0.5006 \\
0.07	0.5006 \\
0.09	0.5006 \\
0.11	0.5006 \\
0.13	0.5006 \\
0.15	0.5006 \\
0.17	0.5006 \\
0.19	0.5006 \\
0.21	0.5006 \\
0.23	0.5006 \\
0.25	0.5006 \\
0.27	0.5006 \\
0.29	0.5006 \\
0.31	0.5006 \\
0.33	0.5006 \\
0.35	0.5006 \\
0.37	0.5006 \\
0.39	0.5006 \\
0.41	0.5006 \\
0.43	0.5006 \\
0.45	0.5006 \\
0.47	0.5006 \\
0.49	0.5006 \\
0.51	0.5006 \\
0.53	0.5006 \\
0.55	0.5006 \\
0.57	0.5006 \\
0.59	0.5006 \\
0.61	0.5006 \\
0.63	0.5006 \\
0.65	0.5006 \\
0.67	0.5006 \\
0.69	0.5006 \\
0.71	0.5006 \\
0.73	0.5006 \\
0.75	0.5006 \\
0.77	0.5006 \\
0.79	0.5006 \\
0.81	0.5006 \\
0.83	0.5006 \\
0.85	0.5006 \\
0.87	0.5006 \\
0.89	0.5006 \\
0.91	0.5006 \\
0.93	0.5006 \\
0.95	0.5006 \\
0.97	0.5006 \\
0.99	0.5006 \\
};
\addplot [thick, mark=o, lightgray!62.22222222222222!black, solid]
table [row sep=\\]{%
0.00	0.69951 \\
};
\addplot [thick, mark=*, lightgray!62.22222222222222!black, solid]
table [row sep=\\]{%
1.00	0.69951 \\
};
\addplot [thick, lightgray!62.22222222222222!black, mark=triangle, mark repeat=10]
table [row sep=\\]{%
0.01	0.69951 \\
0.03	0.69951 \\
0.05	0.69951 \\
0.07	0.69951 \\
0.09	0.69951 \\
0.11	0.69951 \\
0.13	0.69951 \\
0.15	0.69951 \\
0.17	0.69951 \\
0.19	0.69951 \\
0.21	0.69951 \\
0.23	0.69951 \\
0.25	0.69951 \\
0.27	0.69951 \\
0.29	0.69951 \\
0.31	0.69951 \\
0.33	0.69951 \\
0.35	0.69951 \\
0.37	0.69951 \\
0.39	0.69951 \\
0.41	0.69951 \\
0.43	0.69951 \\
0.45	0.69951 \\
0.47	0.69951 \\
0.49	0.69951 \\
0.51	0.69951 \\
0.53	0.69951 \\
0.55	0.69951 \\
0.57	0.69951 \\
0.59	0.69951 \\
0.61	0.69951 \\
0.63	0.69951 \\
0.65	0.69951 \\
0.67	0.69951 \\
0.69	0.69951 \\
0.71	0.69951 \\
0.73	0.69951 \\
0.75	0.69951 \\
0.77	0.69951 \\
0.79	0.69951 \\
0.81	0.69951 \\
0.83	0.69951 \\
0.85	0.69951 \\
0.87	0.69951 \\
0.89	0.69951 \\
0.91	0.69951 \\
0.93	0.69951 \\
0.95	0.69951 \\
0.97	0.69951 \\
0.99	0.69951 \\
};
\addplot [thick, mark=o, color3, solid]
table [row sep=\\]{%
0.00	0.90006 \\
};
\addplot [thick, mark=*, color3, solid]
table [row sep=\\]{%
1.00	0.90006 \\
};
\addplot [thick, color3, mark=diamond, mark repeat=10]
table [row sep=\\]{%
0.01	0.90006 \\
0.03	0.90006 \\
0.05	0.90006 \\
0.07	0.90006 \\
0.09	0.90006 \\
0.11	0.90006 \\
0.13	0.90006 \\
0.15	0.90006 \\
0.17	0.90006 \\
0.19	0.90006 \\
0.21	0.90006 \\
0.23	0.90006 \\
0.25	0.90006 \\
0.27	0.90006 \\
0.29	0.90006 \\
0.31	0.90006 \\
0.33	0.90006 \\
0.35	0.90006 \\
0.37	0.90006 \\
0.39	0.90006 \\
0.41	0.90006 \\
0.43	0.90006 \\
0.45	0.90006 \\
0.47	0.90006 \\
0.49	0.90006 \\
0.51	0.90006 \\
0.53	0.90006 \\
0.55	0.90006 \\
0.57	0.90006 \\
0.59	0.90006 \\
0.61	0.90006 \\
0.63	0.90006 \\
0.65	0.90006 \\
0.67	0.90006 \\
0.69	0.90006 \\
0.71	0.90006 \\
0.73	0.90006 \\
0.75	0.90006 \\
0.77	0.90006 \\
0.79	0.90006 \\
0.81	0.90006 \\
0.83	0.90006 \\
0.85	0.90006 \\
0.87	0.90006 \\
0.89	0.90006 \\
0.91	0.90006 \\
0.93	0.90006 \\
0.95	0.90006 \\
0.97	0.90006 \\
0.99	0.90006 \\
};
\addplot [thick, mark=o, black]
table [row sep=\\]{%
1.0000000000000000 	0.0 \\
};
\addplot [thick, black]
table [row sep=\\]{%
1.0000000000000000 	0.0 \\
1.3000000000000000 	0.0 \\
};
\end{axis}

\end{tikzpicture}
        \label{fig:bernoulli}}
    \subfigure[Transformed Rewards]{
\begin{tikzpicture}[scale=0.75, every node/.style={scale=0.85}]

\definecolor{color1}{rgb}{0,0,0}
\definecolor{color2}{rgb}{0.75,0.75,0}
\definecolor{color3}{rgb}{0.75,0,0.75}

\begin{axis}[
grid=both,
grid style={white!50!black},
axis background/.style={fill=white},
axis line style={black},
x label style={at={(axis description cs:0.5,-0.025)},anchor=north},
y label style={at={(axis description cs:-0.1,.5)},anchor=south},
legend cell align={left},
legend entries={{$\lambda_i$ = 1.0},{$\lambda_i$ = 3.0},{$\lambda_i$ = 5.0},{$\lambda_i$ = 7.0},{$\lambda_i$ = 9.0}},
legend style={at={(0.03,0.03)}, anchor=south west, draw=white!80.0!black, fill=white!99.80392156862746!black},
tick align=outside,
tick pos=left,
xlabel={x},
xmajorgrids,
xmin=-0.0392780323685732, xmax=1.03278054711694,
ylabel={$P(X \geq x)$},
ymajorgrids,
ymin=-0.049849, ymax=1.048149
]
\addlegendimage{thick, green!50.0!black, mark=asterisk}
\addlegendimage{thick, color1, mark=o}
\addlegendimage{thick, color2, mark=star}
\addlegendimage{thick, lightgray!62.22222222222222!black, mark=triangle}
\addlegendimage{thick, color3, mark=diamond}
\addplot [thick, color3]
table [row sep=\\]{%
-0.3000000000000000	1.0 \\
0.00000000000000000	1.0 \\
0.00998401282922471	0.99824 \\
};
\addplot [thick, green!50.0!black, mark=asterisk, mark repeat=10]
table [row sep=\\]{%
0.00945190306258633	0.98607 \\
0.0283411862958662	0.95711 \\
0.0472304695291461	0.92905 \\
0.0661197527624261	0.90093 \\
0.085009035995706	0.87436 \\
0.103898319228986	0.84834 \\
0.122787602462266	0.82245 \\
0.141676885695546	0.79774 \\
0.160566168928826	0.774 \\
0.179455452162106	0.74955 \\
0.198344735395385	0.72573 \\
0.217234018628665	0.70265 \\
0.236123301861945	0.67912 \\
0.255012585095225	0.65696 \\
0.273901868328505	0.63442 \\
0.292791151561785	0.61227 \\
0.311680434795065	0.58986 \\
0.330569718028345	0.56768 \\
0.349459001261625	0.54503 \\
0.368348284494905	0.52232 \\
0.387237567728185	0.49982 \\
0.406126850961464	0.47765 \\
0.425016134194744	0.45601 \\
0.443905417428024	0.43448 \\
0.462794700661304	0.4118 \\
0.481683983894584	0.38993 \\
0.500573267127864	0.3683 \\
0.519462550361144	0.34625 \\
0.538351833594424	0.32515 \\
0.557241116827704	0.30272 \\
0.576130400060984	0.28038 \\
0.595019683294264	0.2588 \\
0.613908966527543	0.23701 \\
0.632798249760823	0.21562 \\
0.651687532994103	0.19437 \\
0.670576816227383	0.17317 \\
0.689466099460663	0.15287 \\
0.708355382693943	0.13289 \\
0.727244665927223	0.11381 \\
0.746133949160503	0.09557 \\
0.765023232393783	0.07733 \\
0.783912515627063	0.06034 \\
0.802801798860343	0.04527 \\
0.821691082093623	0.03205 \\
0.840580365326902	0.02078 \\
0.859469648560182	0.01216 \\
0.878358931793462	0.00624 \\
0.897248215026742	0.00246 \\
0.916137498260022	0.00058 \\
0.935026781493302	6e-05 \\
};
\addplot [thick, color1, mark=o, mark repeat=10]
table [row sep=\\]{%
0.00984086645688068	0.99506 \\
0.0294771690221534	0.98508 \\
0.0491134715874262	0.97533 \\
0.0687497741526989	0.96462 \\
0.0883860767179717	0.95578 \\
0.108022379283244	0.94562 \\
0.127658681848517	0.93542 \\
0.14729498441379	0.9257 \\
0.166931286979063	0.91591 \\
0.186567589544335	0.90566 \\
0.206203892109608	0.89555 \\
0.225840194674881	0.88516 \\
0.245476497240154	0.87456 \\
0.265112799805426	0.86402 \\
0.284749102370699	0.85299 \\
0.304385404935972	0.84275 \\
0.324021707501245	0.83154 \\
0.343658010066517	0.81964 \\
0.36329431263179	0.80776 \\
0.382930615197063	0.7957 \\
0.402566917762336	0.78312 \\
0.422203220327608	0.7707 \\
0.441839522892881	0.75806 \\
0.461475825458154	0.74432 \\
0.481112128023427	0.73081 \\
0.5007484305887	0.71641 \\
0.520384733153972	0.70157 \\
0.540021035719245	0.68573 \\
0.559657338284518	0.6693 \\
0.579293640849791	0.65193 \\
0.598929943415063	0.6339 \\
0.618566245980336	0.61443 \\
0.638202548545609	0.59399 \\
0.657838851110882	0.57251 \\
0.677475153676154	0.54859 \\
0.697111456241427	0.52421 \\
0.7167477588067	0.49884 \\
0.736384061371973	0.46977 \\
0.756020363937245	0.43848 \\
0.775656666502518	0.40493 \\
0.795292969067791	0.36864 \\
0.814929271633064	0.32917 \\
0.834565574198336	0.2861 \\
0.854201876763609	0.23998 \\
0.873838179328882	0.19199 \\
0.893474481894155	0.13962 \\
0.913110784459427	0.08833 \\
0.9327470870247	0.04266 \\
0.952383389589973	0.01134 \\
0.972019692155246	0.00057 \\
};
\addplot [thick, color2, mark=star, mark repeat=10]
table [row sep=\\]{%
0.00991417808320172	0.99674 \\
0.0297218639285136	0.9909 \\
0.0495295497738254	0.98468 \\
0.0693372356191373	0.97844 \\
0.0891449214644491	0.97251 \\
0.108952607309761	0.96658 \\
0.128760293155073	0.96024 \\
0.148567979000385	0.95376 \\
0.168375664845696	0.94739 \\
0.188183350691008	0.94103 \\
0.20799103653632	0.93441 \\
0.227798722381632	0.92858 \\
0.247606408226944	0.9217 \\
0.267414094072256	0.91485 \\
0.287221779917568	0.90781 \\
0.307029465762879	0.90085 \\
0.326837151608191	0.89363 \\
0.346644837453503	0.88638 \\
0.366452523298815	0.87898 \\
0.386260209144127	0.87102 \\
0.406067894989439	0.8626 \\
0.425875580834751	0.85457 \\
0.445683266680062	0.84587 \\
0.465490952525374	0.83615 \\
0.485298638370686	0.82671 \\
0.505106324215998	0.8169 \\
0.52491401006131	0.80629 \\
0.544721695906622	0.79497 \\
0.564529381751933	0.78293 \\
0.584337067597245	0.77004 \\
0.604144753442557	0.75672 \\
0.623952439287869	0.74219 \\
0.643760125133181	0.72613 \\
0.663567810978493	0.70999 \\
0.683375496823804	0.69261 \\
0.703183182669116	0.67255 \\
0.722990868514428	0.65116 \\
0.74279855435974	0.6276 \\
0.762606240205052	0.60074 \\
0.782413926050364	0.57192 \\
0.802221611895676	0.53881 \\
0.822029297740987	0.50087 \\
0.841836983586299	0.45653 \\
0.861644669431611	0.4072 \\
0.881452355276923	0.34701 \\
0.901260041122235	0.27766 \\
0.921067726967547	0.19995 \\
0.940875412812859	0.11625 \\
0.96068309865817	0.03954 \\
0.980490784503482	0.00154 \\
};
\addplot [thick, lightgray!62.22222222222222!black, mark=triangle, mark repeat=10]
table [row sep=\\]{%
0.0099869622251715	0.99777 \\
0.0298242967855	0.99339 \\
0.0496616313458284	0.98895 \\
0.0694989659061569	0.98451 \\
0.0893363004664853	0.98028 \\
0.109173635026814	0.97548 \\
0.129010969587142	0.97105 \\
0.148848304147471	0.96681 \\
0.168685638707799	0.96205 \\
0.188522973268128	0.95734 \\
0.208360307828456	0.95271 \\
0.228197642388784	0.94837 \\
0.248034976949113	0.9432 \\
0.267872311509441	0.93801 \\
0.28770964606977	0.93284 \\
0.307546980630098	0.92755 \\
0.327384315190427	0.92242 \\
0.347221649750755	0.9166 \\
0.367058984311084	0.91064 \\
0.386896318871412	0.90483 \\
0.40673365343174	0.89891 \\
0.426570987992069	0.8923 \\
0.446408322552397	0.88565 \\
0.466245657112726	0.87856 \\
0.486082991673054	0.87174 \\
0.505920326233383	0.86365 \\
0.525757660793711	0.85498 \\
0.54559499535404	0.84607 \\
0.565432329914368	0.83715 \\
0.585269664474697	0.828 \\
0.605106999035025	0.81727 \\
0.624944333595353	0.80657 \\
0.644781668155682	0.79439 \\
0.66461900271601	0.7814 \\
0.684456337276339	0.76692 \\
0.704293671836667	0.75089 \\
0.724131006396996	0.73354 \\
0.743968340957324	0.71369 \\
0.763805675517653	0.69118 \\
0.783643010077981	0.66632 \\
0.803480344638309	0.63741 \\
0.823317679198638	0.60366 \\
0.843155013758966	0.56455 \\
0.862992348319295	0.5178 \\
0.882829682879623	0.46319 \\
0.902667017439952	0.39583 \\
0.92250435200028	0.31094 \\
0.942341686560609	0.20741 \\
0.962179021120937	0.09169 \\
0.982016355681265	0.00641 \\
};
\addplot [thick, color3, mark=diamond, mark repeat=10]
table [row sep=\\]{%
0.00998401282922471	0.99824 \\
0.0298629230099707	0.99473 \\
0.0497418331907166	0.99131 \\
0.0696207433714626	0.9878 \\
0.0894996535522086	0.9844 \\
0.109378563732955	0.98113 \\
0.129257473913701	0.97775 \\
0.149136384094446	0.9742 \\
0.169015294275192	0.97066 \\
0.188894204455938	0.96693 \\
0.208773114636684	0.96283 \\
0.22865202481743	0.95912 \\
0.248530934998176	0.95516 \\
0.268409845178922	0.95114 \\
0.288288755359668	0.94741 \\
0.308167665540414	0.94328 \\
0.32804657572116	0.93909 \\
0.347925485901906	0.93435 \\
0.367804396082652	0.92959 \\
0.387683306263398	0.92491 \\
0.407562216444144	0.92045 \\
0.42744112662489	0.91511 \\
0.447320036805636	0.90949 \\
0.467198946986382	0.90412 \\
0.487077857167128	0.89888 \\
0.506956767347874	0.89301 \\
0.52683567752862	0.88658 \\
0.546714587709366	0.87948 \\
0.566593497890112	0.87204 \\
0.586472408070858	0.86381 \\
0.606351318251604	0.85548 \\
0.62623022843235	0.84585 \\
0.646109138613096	0.83554 \\
0.665988048793842	0.82507 \\
0.685866958974587	0.81365 \\
0.705745869155334	0.80056 \\
0.725624779336079	0.78588 \\
0.745503689516825	0.76848 \\
0.765382599697571	0.74935 \\
0.785261509878317	0.72813 \\
0.805140420059063	0.70328 \\
0.825019330239809	0.67426 \\
0.844898240420555	0.63972 \\
0.864777150601301	0.59785 \\
0.884656060782047	0.5459 \\
0.904534970962793	0.4805 \\
0.924413881143539	0.39376 \\
0.944292791324285	0.28172 \\
0.964171701505031	0.1398 \\
0.984050611685777	0.01183 \\
};
\addplot[thick, color3]
table[row sep=\\]{
0.984050611685777	0.01183 \\
0.989050611685777	0.00607 \\
0.994019692155246	0.00307 \\
0.999019692155246	0.00207 \\
0.999519692155246	0.00127 \\
1.00000000000000	0.00000 \\
1.300000000000000	0.00000 \\
};
\end{axis}

\end{tikzpicture}
        \label{fig:exponential}}
    \caption{First Order Stochastic Dominance on reward distribution}\label{fig:majorisation}
\end{figure*}

We consider two different types of reward function to evaluate the performance of \NAM. The two reward functions are the sum of rewards of each arm, and the maximum of rewards of each arm, In the following, we will describe the functions and evaluate them for one of the distributions of the rewards from each arm.


\subsection{Sum of Bernoulli arm rewards} \label{sum_rewards}
Assume that a company wishes to do a campaign of the product and chooses $K$ out of available $N$ sub-campaigns each day. Let the reward of  sub-campaign $i$, $X_i$, be Bernoulli with parameter $p_i$, which is unknown. Further, let the reward that the company receives is how the overall company sales progressed, thus receiving the aggregate reward as $\sum_{i = 1}^{K}X_i$. To normalize the received reward, we let $r = \frac{1}{K}\sum_{i = 1}^{K}X_i$. We assume that the individual arm rewards are not observed by the company, while the overall progress of the campaign can be seen. Another application is showing $K$ out of $N$ advertisements to the user webpage, where the reward is in a form of whether user clicks the ad to go to the product page. The aggregate reward is the total number of clicks, and we assume that individual click information is not available.


Thus, we have the expected reward as 
\begin{eqnarray*}
    \mathbb{E}[r] &=& \mathbb{E}\left[\frac{1}{K}\sum_{i = 1}^KX_i\right]\\
    &=& \frac{1}{K}\sum_{i = 1}^K\mathbb{E}\left[X_i\right]\\
    &=& \frac{1}{K}\sum_{i = 1}^Kp_i
\end{eqnarray*}
Since the combined reward is sum of all individual rewards, the expected reward is strictly increasing with respect to the expected rewards of the individual arms. Prior works in click optimization assume knowledge of clicks on individual advertisements and thus are semi-bandits \citep{kveton2018bubblerank}. 

Recently \citep{rejwan2020top} proposed CSAR algorithm to find top $K$ arms which solves this problem setup efficiently using Hadamard matrices. CSAR algorithm, similar to ours, divide the set of arms into groups of size $2K$. The algorithm, then estimates the individual arm rewards of the $2K$ arms in each group using Hadamard matrices of size $2K$. Each column maps to an arm in the group and for each row, an action can now be constructed by selected all arms corresponding to $+1$ elements in the row or the arms corresponding to $-1$ elements in the row.

We compare our \NAM\ with UCB algorithm and CSAR algorithm for $N = 45$ and various values of $K$ for $T= 10^6$ steps. For CSAR algorithm we select $K \in \{2, 4, 8\}$ for easy construction of Hadamard matrices of size $2K$. For UCB algorithm we select $K\in\{2, 4\}$ for easy construction of the combinatorial action space of ${{N}\choose{K}}$. We plot average cumulative regret over $25$ independent runs for fixed value of individual arm rewards sampled randomly uniformly from $[0,1]$.

Figure \ref{fig:lin_bandit_comparison} shows the comparison results between \NAM, CSAR algorithm, and UCB algorithm. We note that for $K=2$, UCB algorithm outperforms both CSAR algorithm and \NAM\ algorithm. This is because for small values of $K$, UCB can explore all $\binom{N}{K}$ actions faster as UCB does not estimate individual arm rewards as CSAR algorithm, and UCB can eliminate sub optimal arms from direct comparison with the best arm unlike \NAM\ algorithm. We note that in the cumulative regret of CSAR algorithm rises sharply for $K=2$ in Figure \ref{fig:linbandit_K_2}. We suspect that this is because of a large number of samples might be required to eliminate a sub-optimal arm. Since our implementation of the algorithm samples a group consecutively for $m = 1/2^r$ times in round $r$, we see a jump instead of a smooth rise. Such behaviour in not visible for $K=4$ or $K=8$ because we suspect that the elimination of sub-optimal arms might have been quicker compared for the case of $K=2$ because of the large gap between arms.

Also we note that \NAM\ does not perform as good as CSAR algorithm. We note that this difference is because CSAR estimates expected rewards of individual arms. CSAR algorithm can construct optimal action from the estimates and pull the optimal arm more frequently. \NAM\ eliminates arms slowly compared to CSAR algorithm, as CSAR compares each arm from the $K^{th}$ best arm which CMAB cannot perform. However, we note that even for $K=4$, both CSAR and \NAM\ algorithm outperform UCB algorithm by significant margin. Also, for $K=8$, we note that UCB algorithm would not even finish exploring ${{45}\choose{8}} = 215553195$ arms in $10^6$ time steps. Hence a comparison with UCB for $K=8$ is futile.



\begin{figure}[!htbp]
     \centering
        \subfigure[$K=2$]{
\begin{tikzpicture}[thick,scale=0.57, every node/.style={scale=0.85}]

\definecolor{color0}{rgb}{0,0.749019607843137,1}

\begin{axis}[
legend cell align={left},
legend style={at={(0.03,0.97)}, anchor=north west, draw=white!80.0!black},
tick align=outside,
tick pos=left,
x grid style={white!69.01960784313725!black},
xlabel={\(\displaystyle t\)},
xmin=-48000, xmax=1008000,
xtick style={color=black},
y grid style={white!69.01960784313725!black},
ylabel={\(\displaystyle W(t)\)},
ymin=-22714.8785, ymax=136115.6685,
ytick style={color=black}
]
\addplot [semithick, red]
table {%
0 0.36
40000 19943.1
80000 36254.44
120000 48843.56
160000 59965.46
200000 67170.12
240000 69961.86
280000 74167.4
320000 79565.36
360000 84691.22
400000 89588.54
440000 93501.28
480000 97193.46
520000 100907
560000 104605.58
600000 108321.72
640000 112028.44
680000 115665.1
720000 117314.8
760000 118138.52
800000 118953.9
840000 119767.68
880000 120593.62
920000 121425.24
960000 122253.74
};
\addlegendentry{CSAR}
\addplot [semithick, green!50.19607843137255!black]
table {%
0 0.96
40000 21387.38
80000 39291.42
120000 55676.84
160000 80141.76
200000 87219.44
240000 90726.92
280000 95090.58
320000 98931.84
360000 102900.04
400000 105300.3
440000 108701.58
480000 111584.86
520000 114232.88
560000 115663.8
600000 116576.42
640000 120516.44
680000 123064.86
720000 124044.78
760000 125468.52
800000 127107.62
840000 128476.94
880000 129585.72
920000 129621.22
960000 129634.14
};
\addlegendentry{CMAB-SM}
\addplot [semithick, color0]
table {%
0 0.96
40000 19668.2
80000 34446.3
120000 45412.42
160000 50773.2
200000 53686.92
240000 55285.84
280000 56366.06
320000 56897.2
360000 57362.6
400000 57639.4
440000 57902.92
480000 58161.68
520000 58327
560000 58476.78
600000 58610.98
640000 58751.52
680000 58880.76
720000 58977.06
760000 58979.58
800000 58979.2
840000 58979.94
880000 58981
920000 58978.2
960000 58980.04
};
\addlegendentry{UCB}
\end{axis}

\end{tikzpicture}
            \label{fig:linbandit_K_2}}
         \subfigure[$K=4$]{
\begin{tikzpicture}[thick,scale=0.57, every node/.style={scale=0.85}]

\definecolor{color0}{rgb}{0,0.749019607843137,1}

\begin{axis}[
legend cell align={left},
legend style={at={(0.03,0.97)}, anchor=north west, draw=white!80.0!black},
tick align=outside,
tick pos=left,
x grid style={white!69.01960784313725!black},
xlabel={\(\displaystyle t\)},
xmin=-48000, xmax=1008000,
xtick style={color=black},
y grid style={white!69.01960784313725!black},
ylabel={\(\displaystyle W(t)\)},
ymin=-22714.8785, ymax=477023.6685,
ytick style={color=black}
]
\addplot [semithick, red]
table {%
0 0.51
40000 18825.05
80000 35080.59
120000 44838.93
160000 56978.27
200000 68812.59
240000 75249.7
280000 80097.09
320000 84949.89
360000 89496.61
400000 91117.94
440000 92749.29
480000 94368.29
520000 96030.72
560000 97738.42
600000 99440.74
640000 101140.65
680000 102840.61
720000 104541.8
760000 106244.36
800000 107943.82
840000 109650.46
880000 111352.03
920000 113052.05
960000 114750.51
};
\addlegendentry{CSAR}
\addplot [semithick, green!50.19607843137255!black]
table {%
0 0.95
40000 18202.18
80000 39689.28
120000 58342.65
160000 73116.98
200000 90055.97
240000 114408.43
280000 127794.6
320000 140785.04
360000 152529.68
400000 159147.93
440000 165137.51
480000 170778.36
520000 175258.83
560000 180015.43
600000 184436.66
640000 187943.66
680000 190975.17
720000 193855.76
760000 196122.7
800000 202387.47
840000 206025.1
880000 206424.87
920000 206637.59
960000 209294.67
};
\addlegendentry{CMAB-SM}
\addplot [semithick, color0]
table {%
0 0.95
40000 18490
80000 38560.99
120000 57030.2
160000 76760.84
200000 94637.5
240000 113776.61
280000 132505.07
320000 151684.51
360000 170009.9
400000 189893.51
440000 208105.13
480000 226606.82
520000 246697.14
560000 265182.46
600000 284215.81
640000 302759.58
680000 322395.65
720000 340879.13
760000 360671.74
800000 378039.67
840000 397722.13
880000 416442.41
920000 435154.76
960000 454308.28
};
\addlegendentry{UCB}
\end{axis}

\end{tikzpicture}
            \label{fig:linbandit_K_4}}
       \subfigure[$K=8$]{
\begin{tikzpicture}[thick,scale=0.57, every node/.style={scale=0.85}]

\begin{axis}[
legend cell align={left},
legend style={at={(0.03,0.97)}, anchor=north west, draw=white!80.0!black},
tick align=outside,
tick pos=left,
x grid style={white!69.01960784313725!black},
xlabel={\(\displaystyle t\)},
xmin=-48000, xmax=1008000,
xtick style={color=black},
y grid style={white!69.01960784313725!black},
ylabel={\(\displaystyle W(t)\)},
ymin=-11643.78425, ymax=244530.24925,
ytick style={color=black}
]
\addplot [semithick, red]
table {%
0 0.49
40000 17096.235
80000 32244.37
120000 42530.285
160000 54459.535
200000 67948.49
240000 78335.88
280000 87414.525
320000 91819.85
360000 95699.665
400000 99578.1
440000 103464.395
480000 107351.91
520000 110924.015
560000 114162.72
600000 116818.89
640000 119194.945
680000 121573.405
720000 123941.585
760000 126322.44
800000 128708.02
840000 131088.43
880000 133402.8
920000 135670.625
960000 137932.66
};
\addlegendentry{CSAR}
\addplot [semithick, green!50.19607843137255!black]
table {%
0 0.93
40000 17529.04
80000 34397.3
120000 49857.145
160000 64316.095
200000 81479.99
240000 101529.15
280000 119040.44
320000 135030.195
360000 150326.3
400000 166032.73
440000 181218.52
480000 192011.97
520000 197721.635
560000 203454.935
600000 209039.47
640000 215157.05
680000 219967.745
720000 221783.72
760000 223580.865
800000 225914.825
840000 228394.665
880000 230462.605
920000 231649.19
960000 232885.975
};
\addlegendentry{CMAB-SM}
\end{axis}

\end{tikzpicture}
            \label{fig:linbandit_K_8}}
        \caption{Comparison results for CMAB-SM and CSAR Algorithm for combinatorial linear bandit reward as mean of individual arm rewards. 
        }
        \label{fig:lin_bandit_comparison}
\end{figure}

\subsection{Maximum of Bernoulli rewards} \label{max_rewards}
We consider a case where agent is a recommendation system that shows a list of restaurants or hotels, and user provides feedback whether or not the list is useful. A user finds the list useful when she is able to get a recommendation suiting her requirements. We take the reward of individual arm to be discrete with value $1$ if the item was useful, and $0$ for the case where the item in list is not useful. We assume that the rewards follow Bernoulli distribution. Since the individual rewards are not observed, this is a bandit setting. Further, note that the maximum function is not a linear function. We will now show the strictly increasing property of the function. The expected reward of selecting $K$ arms is given as 

\begin{eqnarray}
    \mathbb{E}[r] &=& \mathbb{E}\left[\max(X_1, X_2, \cdots, X_K)\right] \label{eq:exp_max}\\
    &=& 1\left(1 - P\left({\bigcap_{i = 1}^K\{X_i = 0\}}\right)\right) \nonumber\\&&+ 0\left(P\left(\bigcap_{i = 1}^K\{X_i = 0\}\right)\right) \label{eq:exp_max_inter_X_i}\\
    &=& 1\left(1 -  \prod_{i = 1}^K\left(1 - p_i\right)\right) + 0\left(\prod_{i = 1}^K\left(1 - p_i\right)\right) \label{eq:exp_max_prod_X_i}\\
    &=& 1 -  \prod_{i = 1}^K(1 - p_i),  \label{eq:exp_max_final}
\end{eqnarray}
where  (\ref{eq:exp_max}) is the expected value of the function of individual rewards,  (\ref{eq:exp_max_inter_X_i}) follows from the fact that individual rewards are Bernoulli distributed and their maximum is zero only when all the individual rewards are zero, (\ref{eq:exp_max_prod_X_i}) holds since the rewards are independent of each other.

{\cite{gopalan14thompson} present a unique way to solve this problem in Section 4.2 of of their paper. They consider the success probability $p_i$ of arm $i$ comes from the set $\left\{1-\beta^R, 1-\beta^{R-1}, \cdots, 1-\beta \right\}$ for all $i \in \{1, 2,\cdots, N\}$ and $\beta \in (0,1)$ and $R>0$ are fixed parameters. Additionally, their bounds are of order $O\left({N-1}\choose K\right)$. Our setting is a generalization of the setting considered by \cite{gopalan14thompson}. We let the $p_i$ lie in the set $[0,1]$ and we obtain a bound which is polynomial in $N, K$.}

The reward is non-linear and  the expected value of reward is strictly increasing in expected rewards of individual arms. Cascade model of click optimization by \citep{pmlr-v37-kveton15} uses a similar problem formulation, however they still consider information on clicks on individual items. Figure \ref{fig:Max_of_rewards} shows the evaluation results in this case. We note that for $K\geq3$, the proposed algorithm significantly outperforms UCB. Even for $N=24$ and $K=2$, where $\binom{N}{K} = 220$, $W(T)$ for \NAM\ is close to UCB. 

We note that the cumulative regret at any time $t$ decreases as $K$ increases. This follows from the fact that as with increasing $K$, a user will have more choices at any given time and it is more likely that the arm with the highest reward is in the $K$ chosen arms. 

\begin{figure*}
    \hspace{-0.4cm}
    \subfigure[$N=24$]{
	 \input{max_N_24}
		\label{fig:max_N_24}}
 \subfigure[$N=12$]{
\begin{tikzpicture}[thick,scale=0.85, every node/.style={scale=0.9}]

\definecolor{color0}{rgb}{0.75,0,0.75}

\begin{axis}[
grid=both,
grid style={white!75!black},
axis background/.style={fill=white},
axis line style={black},
x label style={at={(axis description cs:0.5,-0.025)},anchor=north},
y label style={at={(axis description cs:-0.08,.5)},anchor=south},
xlabel={t},
ylabel={W(t)},
ymode=log,
legend cell align={left},
legend entries={{UCB, K=2},{CMAB\_SM, K=2},{UCB, K=3},{CMAB\_SM, K=3},{UCB, K=5},{CMAB\_SM, K=5}},
legend style={at={(0.55,0.37)},nodes={scale=0.6, transform shape}, anchor=north west, draw=white!80.0!black, fill=white!99.80392156862746!black},
tick pos=left,
xmajorgrids,
xmin=-49000, xmax=1029000,
ymajorgrids,
ymin=300, ymax=77665.399
]
\addlegendimage{mark=asterisk, dashed, green!50.0!black}
\addlegendimage{mark=asterisk, green!50.0!black}
\addlegendimage{mark=o, dashed, color0}
\addlegendimage{mark=o, color0}
\addlegendimage{mark=star, dashed, black}
\addlegendimage{mark=star, black}
\addplot [thick, mark=asterisk, mark repeat=10, green!50.0!black, dashed]
table [row sep=\\]{%
0	1 \\
20000	2646.27 \\
40000	3818.63 \\
60000	4617.23 \\
80000	5247.93 \\
100000	5710.27 \\
120000	6084.9 \\
140000	6422.07 \\
160000	6713.67 \\
180000	6960.07 \\
200000	7196.2 \\
220000	7397.57 \\
240000	7577.6 \\
260000	7753.63 \\
280000	7910.67 \\
300000	8061.4 \\
320000	8206.2 \\
340000	8342.87 \\
360000	8471.43 \\
380000	8591.37 \\
400000	8712.97 \\
420000	8819.8 \\
440000	8935.37 \\
460000	9048.17 \\
480000	9155.8 \\
500000	9251.7 \\
520000	9343.1 \\
540000	9428.47 \\
560000	9522.6 \\
580000	9600.67 \\
600000	9678.37 \\
620000	9754.93 \\
640000	9835.6 \\
660000	9913.1 \\
680000	9984.67 \\
700000	10049.1 \\
720000	10110.7 \\
740000	10173.9 \\
760000	10241.1 \\
780000	10300.73 \\
800000	10363.43 \\
820000	10425.97 \\
840000	10478.7 \\
860000	10527.17 \\
880000	10572.83 \\
900000	10616.97 \\
920000	10662.13 \\
940000	10695.23 \\
960000	10737.63 \\
980000	10775 \\
};
\addplot [thick, mark=asterisk, mark repeat=10, green!50.0!black]
table [row sep=\\]{%
0	1 \\
20000	3999.87 \\
40000	7459.87 \\
60000	10012.93 \\
80000	11999.6 \\
100000	14206.37 \\
120000	16997.37 \\
140000	20076.47 \\
160000	23069.83 \\
180000	25108.43 \\
200000	27020 \\
220000	29067.43 \\
240000	30781.43 \\
260000	32099.97 \\
280000	33300.83 \\
300000	34555.43 \\
320000	35873.9 \\
340000	37081.93 \\
360000	38152.1 \\
380000	39161.73 \\
400000	40211.33 \\
420000	40961.47 \\
440000	41564.3 \\
460000	42141.03 \\
480000	42843.83 \\
500000	43426.67 \\
520000	44067 \\
540000	44516.13 \\
560000	44903.63 \\
580000	45282.43 \\
600000	45682.27 \\
620000	46125.83 \\
640000	46489.77 \\
660000	46882.67 \\
680000	47182.53 \\
700000	47423 \\
720000	47645.77 \\
740000	47995.5 \\
760000	48227.83 \\
780000	48381.53 \\
800000	48535.37 \\
820000	48640.5 \\
840000	48780.93 \\
860000	48953.27 \\
880000	49107.17 \\
900000	49203.87 \\
920000	49261.33 \\
940000	49297.83 \\
960000	49323.53 \\
980000	49358.57 \\
};
\addplot [thick, mark=o, mark repeat=10, color0, dashed]
table [row sep=\\]{%
0	1 \\
20000	1954.93 \\
40000	3702.8 \\
60000	5051.67 \\
80000	6355.4 \\
100000	7649.83 \\
120000	8847.73 \\
140000	9804.67 \\
160000	10632.23 \\
180000	11446.5 \\
200000	12251.4 \\
220000	13056 \\
240000	13838.37 \\
260000	14575.87 \\
280000	15301.03 \\
300000	16022.3 \\
320000	16724.1 \\
340000	17389.27 \\
360000	17980.67 \\
380000	18526.7 \\
400000	19057.5 \\
420000	19540.27 \\
440000	19982.97 \\
460000	20415.03 \\
480000	20847.6 \\
500000	21275.67 \\
520000	21696.3 \\
540000	22117.73 \\
560000	22530.73 \\
580000	22951 \\
600000	23375.4 \\
620000	23790.8 \\
640000	24194.4 \\
660000	24597.83 \\
680000	24988.87 \\
700000	25358.87 \\
720000	25716.83 \\
740000	26080.13 \\
760000	26432.03 \\
780000	26769.2 \\
800000	27100.63 \\
820000	27436.87 \\
840000	27766.13 \\
860000	28085.73 \\
880000	28398.13 \\
900000	28692.37 \\
920000	28988.2 \\
940000	29274.77 \\
960000	29565.7 \\
980000	29854.47 \\
};
\addplot [thick, mark=o, mark repeat=10, color0]
table [row sep=\\]{%
0	1 \\
20000	2184.1 \\
40000	4266 \\
60000	5894.97 \\
80000	7568.63 \\
100000	9106.8 \\
120000	10308.63 \\
140000	11334.37 \\
160000	11986.43 \\
180000	12170.8 \\
200000	12396.37 \\
220000	12628.17 \\
240000	12777.43 \\
260000	12827.97 \\
280000	12834.3 \\
300000	12837.03 \\
320000	12840.13 \\
340000	12844.1 \\
360000	12844.6 \\
380000	12842.1 \\
400000	12843.87 \\
420000	12849.33 \\
440000	12853.23 \\
460000	12857.3 \\
480000	12857.7 \\
500000	12862.4 \\
520000	12863.23 \\
540000	12863.1 \\
560000	12865.87 \\
580000	12861.8 \\
600000	12863 \\
620000	12862.07 \\
640000	12864.43 \\
660000	12866.17 \\
680000	12869.53 \\
700000	12869.93 \\
720000	12868.43 \\
740000	12871.97 \\
760000	12871.47 \\
780000	12869.7 \\
800000	12874.87 \\
820000	12881.63 \\
840000	12883.67 \\
860000	12886.33 \\
880000	12896.03 \\
900000	12903.5 \\
920000	12902.33 \\
940000	12905.07 \\
960000	12906.13 \\
980000	12909.17 \\
};
\addplot [thick, mark=star, mark repeat=10, black, dashed]
table [row sep=\\]{%
0	1 \\
20000	666.73 \\
40000	1333.87 \\
60000	1989.2 \\
80000	2641.87 \\
100000	3292.17 \\
120000	3946.47 \\
140000	4597.73 \\
160000	5177.33 \\
180000	5752 \\
200000	6318.17 \\
220000	6891.37 \\
240000	7463.33 \\
260000	8027.47 \\
280000	8601.37 \\
300000	9164.47 \\
320000	9739.87 \\
340000	10314.53 \\
360000	10883.37 \\
380000	11447.17 \\
400000	12014.9 \\
420000	12558.13 \\
440000	13106.27 \\
460000	13637.43 \\
480000	14155.13 \\
500000	14600.4 \\
520000	15034.93 \\
540000	15457.8 \\
560000	15882.37 \\
580000	16310.23 \\
600000	16734.2 \\
620000	17154.1 \\
640000	17568.87 \\
660000	17994.6 \\
680000	18418.47 \\
700000	18838.27 \\
720000	19254.63 \\
740000	19668.33 \\
760000	20085.87 \\
780000	20506.43 \\
800000	20924.77 \\
820000	21344.73 \\
840000	21763.53 \\
860000	22180.9 \\
880000	22598.63 \\
900000	23011.03 \\
920000	23409.73 \\
940000	23802.27 \\
960000	24203.2 \\
980000	24605.9 \\
};
\addplot [thick, mark=star, mark repeat=10, black]
table [row sep=\\]{%
0	1 \\
20000	631.37 \\
40000	1186.53 \\
60000	1730.47 \\
80000	2172.97 \\
100000	2496.07 \\
120000	2806.03 \\
140000	3216.93 \\
160000	3514.87 \\
180000	3863.1 \\
200000	4178.87 \\
220000	4267.6 \\
240000	4268.03 \\
260000	4268.73 \\
280000	4269 \\
300000	4270.33 \\
320000	4271.33 \\
340000	4272.83 \\
360000	4273.73 \\
380000	4274.97 \\
400000	4277.5 \\
420000	4276.73 \\
440000	4275.53 \\
460000	4274.4 \\
480000	4274.37 \\
500000	4275.73 \\
520000	4276.93 \\
540000	4276.93 \\
560000	4278.83 \\
580000	4278.23 \\
600000	4280.6 \\
620000	4280.43 \\
640000	4280.2 \\
660000	4280.77 \\
680000	4279.83 \\
700000	4280.37 \\
720000	4281.37 \\
740000	4281.8 \\
760000	4283.07 \\
780000	4284.17 \\
800000	4284.9 \\
820000	4285.6 \\
840000	4286.67 \\
860000	4285.1 \\
880000	4283.57 \\
900000	4282.03 \\
920000	4283.43 \\
940000	4283.2 \\
960000	4283.67 \\
980000	4284.07 \\
};
\end{axis}

\end{tikzpicture}
		\label{fig:max_N_12}}
	\caption{Reward of actions is the maximum of rewards of individual arms as described in section \ref{max_rewards}}\label{fig:Max_of_rewards}
\end{figure*}



\end{document}